\documentclass[12pt]{article}
\usepackage{fullpage}

\usepackage{enumitem}
\usepackage[colorlinks,linkcolor=blue,citecolor=blue]{hyperref}
\usepackage[round]{natbib}
\usepackage{bm}
\usepackage[ruled,noline]{algorithm2e}

\usepackage{tikz}
\usetikzlibrary{calc}

\title{Online Learning with Feedback Graphs: \\ Beyond Bandits}

\author{%
Noga Alon%
\footnote{Tel Aviv University, Tel Aviv, Israel, and Microsoft Research, Herzliya, Israel, \email{nogaa@post.tau.ac.il}.}
\and
Nicol\`{o} Cesa-Bianchi%
\footnote{Dipartimento di Informatica, Universit\`{a} degli Studi di Milano, Milan, Italy, \email{nicolo.cesa- bianchi@unimi.it}. Parts of this work were done while the author was at Microsoft Research, Redmond.}
\and
Ofer Dekel%
\footnote{Microsoft Research, Redmond, Washington; \email{oferd@microsoft.com}.}
\and
Tomer Koren%
\footnote{Technion---Israel Institute of Technology, Haifa, Israel, and Microsoft Research, Herzliya, Israel, \email{tomerk@technion.ac.il}. Parts of this work were done while the author was at Microsoft Research, Redmond.}
}

\usepackage{amsmath,amsfonts,amssymb}
\usepackage{mathtools}
\usepackage{amsthm} 

\usepackage[capitalise]{cleveref}
\usepackage{xcolor}
\usepackage{dsfont}

\newcommand{\ignore}[1]{}

\newcommand{\email}[1]{\texttt{#1}}

\theoremstyle{plain}
\newtheorem{theorem}{Theorem}
\newtheorem{lemma}[theorem]{Lemma}

\newtheorem{claim}[theorem]{Claim}

\newtheorem*{theorem*}{Theorem}
\newtheorem*{lemma*}{Lemma}
\newtheorem*{corollary*}{Corollary}
\newtheorem*{proposition*}{Proposition}
\newtheorem*{claim*}{Claim}
\newtheorem*{fact*}{Fact}
\newtheorem*{observation*}{Observation}

\theoremstyle{definition}

\newtheorem*{definition*}{Definition}
\newtheorem*{remark*}{Remark}
\newtheorem*{example*}{Example}

 \theoremstyle{plain}
\newtheorem*{theoremaux}{\theoremauxref}
\gdef\theoremauxref{1}

\newenvironment{repthm}[2][]{%
  \def\theoremauxref{\cref{#2}}
  \begin{theoremaux}[#1]
}{%
  \end{theoremaux}
}

\DeclareMathAlphabet{\mathbfsf}{\encodingdefault}{\sfdefault}{bx}{n}

\let\Pr\relax
\DeclareMathOperator{\Pr}{\mathbb{P}}

\newcommand{\mycases}[4]{{
\left\{
\begin{array}{ll}
    {#1} & {\;\text{#2}} \\[1ex]
    {#3} & {\;\text{#4}}
\end{array}
\right. }}

\newcommand{\mythreecases}[6] {{
\left\{
\begin{array}{ll}
    {#1} & {\;\text{#2}} \\[1ex]
    {#3} & {\;\text{#4}} \\[1ex]
    {#5} & {\;\text{#6}}
\end{array}
\right. }}

\newcommand{\lr}[1]{\!\left(#1\right)\!}
\newcommand{\lrbig}[1]{\big(#1\big)}
\newcommand{\lrBig}[1]{\Big(#1\Big)}
\newcommand{\lrbra}[1]{\!\left[#1\right]\!}

\newcommand{\lrset}[1]{\left\{#1\right\}}

\newcommand{\set}[1]{\{#1\}}
\newcommand{\abs}[1]{|#1|}

\newcommand{\ceil}[1]{\lceil #1 \rceil}
\newcommand{\floor}[1]{\lfloor #1 \rfloor}

\newcommand{\wt}[1]{\smash{\widetilde{#1}}}
\newcommand{\wh}[1]{\smash{\widehat{#1}}}
\renewcommand{\O}{\mathcal{O}}

\newcommand{\E}{\mathbb{E}}
\newcommand{\EE}[1]{\E\lrbra{#1}}

\newcommand{\ind}[1]{\mathbb{I}\!\lrset{#1}}

\newcommand{\st}{\star}

\newcommand{\kl}[2]{\mathsf{D}_\mathrm{KL}(#1,#2)}
\newcommand{\tv}[2]{\mathsf{D}_{\mathrm{TV}}(#1,#2)}

\newcommand{\eps}{\epsilon}

\newcommand{\del}{\delta}

\newcommand{\half}{\frac{1}{2}}
\newcommand{\thalf}{\tfrac{1}{2}}

\newcommand{\eq}{~=~}
\renewcommand{\leq}{~\le~}
\renewcommand{\geq}{~\ge~}

\let\oldtfrac\tfrac
\renewcommand{\tfrac}[2]{\smash{\oldtfrac{#1}{#2}}}

\newcommand{\Q}{\mathcal{Q}}
\newcommand{\F}{\mathcal{F}}
\newcommand{\nin}{N^{\mathrm{in}}}
\newcommand{\nout}{N^{\mathrm{out}}}
\newcommand{\ellh}{\wh{\ell}}

\begin{document}

\maketitle


\begin{abstract} 
We study a general class of online learning problems where the
feedback is specified by a graph. This class 
includes online prediction with expert advice and the multi-armed bandit
problem, but also several learning problems where the online player
does not necessarily observe his own loss. We analyze how the structure of the
feedback graph controls the inherent difficulty of the induced
$T$-round learning problem. Specifically, we show that any feedback graph belongs 
to one of three classes: \emph{strongly observable} graphs,
\emph{weakly observable} graphs, and \emph{unobservable} graphs. We
prove that the first class induces learning problems with
$\wt\Theta(\alpha^{1/2} T^{1/2})$ minimax regret, where $\alpha$ is the
independence number of the underlying graph; the second class induces problems
with $\wt\Theta(\delta^{1/3}T^{2/3})$ minimax regret, where $\delta$
is the domination number of a certain portion of the graph; and the third class
induces problems with linear minimax regret. Our results subsume much of  
the previous work on learning with feedback graphs and reveal new connections to
partial monitoring games. We also show how the regret is affected 
if the graphs are allowed to vary with
time.
\end{abstract}


\clearpage
\section{Introduction}
\label{s:intro}
Online learning can be formulated as a repeated game between a
randomized player and an arbitrary, possibly adversarial, environment
\citep[see, e.g.,][]{cbl06,shalev2011online}. We focus on the
version of the game where, on each round, the player chooses one of
$K$ actions and incurs a corresponding loss. The loss associated with
each action on each round is a number between $0$ and $1$, assigned in
advance by the environment. The player's performance is measured using
the game-theoretic notion of regret, which is the difference between
his cumulative loss and the cumulative loss of the best fixed action
in hindsight. We say that the player is \emph{learning} if his regret
after $T$ rounds is $o(T)$.

After choosing an action, the player observes some feedback, which
enables him to learn and improve his choices on subsequent rounds.  A
variety of different feedback models are discussed in online
learning. The most common is \emph{full feedback}, where the player
gets to see the loss of all the actions at the end of each round. This
feedback model is often called \emph{prediction with expert advice}
\citep{cb+97,LittlestoneWa94,vo90}.  For example, imagine a
single-minded stock market investor who invests all of his wealth in
one of $K$ stocks on each day. At the end of the day, the investor
incurs the loss associated with the stock he chose, but he also
observes the loss of all the other stocks.

Another common feedback model is \emph{bandit feedback}
\citep{AuerCeFrSc02}, where the player only observes the loss of the
action that he chose. In this model, the player's choices influence
the feedback that he receives, so he has to balance an
exploration-exploitation trade-off. On one hand, the player wants to
exploit what he has learned from the previous rounds by choosing an
action that is expected to have a small loss; on the other hand, he
wants to explore by choosing an action that will give him the most
informative feedback.  The canonical example of online learning with
bandit feedback is online advertising.  Say that we operate an
Internet website and we present one of $K$ ads to each user that views
the site. Our goal is to maximize the number of clicked ads and
therefore we incur a unit loss whenever a user doesn't click on an
ad. We know whether or not the user clicked on the ad we presented,
but we don't know whether he would have clicked on any of the
other ads.

Full feedback and bandit feedback are special cases of a general
framework introduced by \citet{MS11}, where the feedback model is
specified by a \emph{feedback graph}. A feedback graph is a directed
graph whose nodes correspond to the player's $K$ actions. A directed
edge from action $i$ to action $j$ (when $i = j$ this edge is called a
\emph{self-loop}) indicates that whenever the player chooses action
$i$ he gets to observe the loss associated with action $j$. The full
feedback model is obtained by setting the feedback graph to be the
directed clique (including all self-loops, see
\cref{fig:graphs}a). The bandit feedback model is obtained by the
graph that only includes the self-loops (see
\cref{fig:graphs}b). Feedback graphs can describe many
other interesting online learning scenarios, as discussed below.

Our main goal is to understand how the structure of the feedback graph
controls the inherent difficulty of the induced online learning
problem. While regret measures the performance of a specific player or
algorithm, the inherent difficulty of the game itself is measured by
the \emph{minimax regret}, which is the regret incurred by an optimal
player that plays against the worst-case environment.  \citet{FS97}
proves that the minimax regret of the full feedback game is
$\Theta(\sqrt{T \ln{K}})$ while \citet{AuerCeFrSc02} proves that the
minimax regret of the bandit feedback game is $\wt\Theta(\sqrt{KT})$.
Both of these settings correspond to feedback graphs where all of the
vertices have self-loops ---we say that the player in these settings
is \emph{self-aware}: he observes his own loss value on each round.
The minimax regret rates induced by self-aware feedback graphs were
extensively studied in \citet{AlonCGMMS14}. In
this paper, we focus on the intriguing situation that occurs
when the feedback graph is missing some self-loops, namely, when the
player does not always observe his own loss. He is still accountable
for the loss on each round, but he does not always know how much loss
he incurred. As revealed by our analysis, the absence of self-loops
can have a significant impact on the minimax regret of the induced
game.

An example of a concrete setting where the player is not always
self-aware is the \emph{apple tasting} problem \citep{helmbold2000apple}. In this problem, the
player examines a sequence of apples, some of which may be rotten. For
each apple, he has two possible actions: he can either discard the
apple (action 1) or he can ship the apple to the market (action
2). The player incurs a unit loss whenever he discards a good apple
and whenever he sends a rotten apple to the market. However, the
feedback is asymmetric: whenever the player chooses to discard an
apple, he first tastes the apple and obtains full feedback; on the
other hand, whenever he chooses to send the apple to the market, he
doesn't taste it and receives no feedback at all. The feedback graph
that describes the apple tasting problem is shown in \cref{fig:graphs}d.
Another problem that is closely related to apple tasting is the
\emph{revealing action} or \emph{label efficient} problem \citep[Example~6.4]{cbl06}. In this problem, one action is a
special action, called the revealing action, which incurs a constant
unit loss. Whenever the player chooses the revealing action, he
receives full feedback. Whenever the player chooses any other action,
he observes no feedback at all (see \cref{fig:graphs}e).

Yet another interesting example where the player is not self-aware is
obtained by setting the feedback graph to be the \emph{loopless
  clique} (the directed clique minus the self-loops, see
\cref{fig:graphs}c). This problem is the complement to the bandit
problem: when the player chooses an action, he observes the loss of
all the other actions, but he does not observe his own loss. To
motivate this, imagine a police officer who wants to prevent crime.
On each day, the officer chooses to stand in one of $K$ possible
locations. Criminals then show up at some of these locations: if a
criminal sees the officer, he runs away before being noticed and the
crime is prevented; otherwise, he goes ahead with the crime.  The
officer gets a unit reward for each crime he prevents,\footnote{It is
  easier to describe this example in terms of maximizing rewards,
  rather than minimizing losses. In our formulation of the problem, a
  reward of $r$ is mathematically equivalent to a loss of $1-r$.} and
at the end of each day he receives a report of all the crimes that
occurred that day. By construction, the officer does not know if his
presence prevented a planned crime, or if no crime was planned for
that location. In other words, the officer observes everything but his
own reward.

Our main result is a full characterization of the minimax regret of
online learning problems defined by feedback graphs. Specifically, we
categorize the set of all feedback graphs into three distinct sets.
The first is the set of \emph{strongly observable} feedback graphs,
which induce online learning problems whose minimax regret is
$\wt\Theta(\alpha^{1/2} T^{1/2})$, where $\alpha$ is the independence
number of the feedback graph. This slow-growing minimax regret rate
implies that the problems in this category are easy to learn. The set
of strongly observable feedback graphs includes the set of self-aware
graphs, so this result extends the characterization given in
\citet{AlonCGMMS14}. The second category is the set of
\emph{weakly observable} feedback graphs, which induce learning
problems whose minimax regret is $\wt\Theta(\delta^{1/3} T^{2/3})$,
where $\delta$ is a new graph-dependent quantity called the weak
domination number of the feedback graph. The minimax regret of these
problems grows at a faster rate of $T^{2/3}$ with the number of
rounds, which implies that the induced problems are hard to learn. The
third category is the set of \emph{unobservable} graphs, which induce
unlearnable $\Theta(T)$ online problems.

Our characterization bears some surprising implications. For example,
the minimax regret for the loopless clique is the same, up to constant
factors, as the $\Theta(\sqrt{T\ln{K}})$ minimax regret for the full feedback
graph. However, if we start with the full feedback graph (the directed
clique with self-loops) and remove a self-loop and an incoming edge
from any node (see \cref{fig:graphs}f), we are left with a
weakly observable feedback graph, and the minimax regret jumps to
order $T^{2/3}$. Another interesting property of our characterization
is how the two learnable categories of feedback graphs depend on
completely different graph-theoretic quantities: the independence
number $\alpha$ and the weak domination number $\delta$.

The setting of online learning with feedback graphs is closely related 
to the more general setting of partial
monitoring \citep[see, e.g.,][Section~6.4]{cbl06}, where the player's
feedback is specified by a feedback matrix, rather than a feedback
graph. Partial monitoring games have also been categorized into three
classes: easy problems with $T^{1/2}$ regret, hard problems with
$T^{2/3}$ regret, and unlearnable problems with linear
regret~\citep[Theorem~2]{bartok2014partial}. If the loss values are
chosen from a finite set (say $\set{0,1}$), then bandit feedback, apple
tasting feedback, and the revealing action feedback models are all
known to be special cases of partial monitoring. In fact, in
\cref{s:partial} we show that any problem in our setting (with
binary losses) can be reduced to the partial
monitoring setting. Nevertheless, the characterization presented in
this paper has several clear advantages over the more general
characterization of partial monitoring games. First, our regret bounds
are minimax optimal not only with respect to $T$, but also with
respect to the other relevant problem parameters.
Second, we obtain our upper bounds with a simple and efficient algorithm. 
Third, our characterization is stated in terms of simple and intuitive
combinatorial properties of the problem.



The paper is organized as follows. In \cref{sec:setting} we define the
problem setting and state our main results. In \cref{s:algo} we
describe our player algorithm and prove upper bounds on the minimax
regret. In \cref{s:lower} we prove matching lower bounds on the
minimax regret.  Finally, in \cref{s:exte} we extend our analysis to
the case where the feedback graph is neither fixed nor known in
advance.

\begin{figure}[t]
\begin{center}
  \begin{tabular}{ccc}
    \begin{tikzpicture}[scale=1, transform shape]
      \draw[white] (-2,-1.7) -- (-2,2.3) -- (2,2.3) -- (2,-1.7) -- (-2,-1.7);
      \foreach \x [evaluate={\l=162-\x*72;}] in {1,...,5} 
               {
                 \node[draw,circle,thick,black,minimum size=0.7cm] (n\x) at ($(0,0)+(\l:1.5)$) {};
                 \node at (n\x) {\x};
               }

               \path[->,black,thick] (n1) edge[anchor=center,loop above] (n1);
               \path[->,black,thick] (n2) edge[anchor=center,loop right] (n2);
               \path[->,black,thick] (n3) edge[anchor=center,loop right] (n3);
               \path[->,black,thick] (n4) edge[anchor=center,loop left] (n4);
               \path[->,black,thick] (n5) edge[anchor=center,loop left] (n5);

               \foreach \x in {1,...,5} \foreach \y in {1,...,5}
                        {
                          \ifnum \x<\y 
                          \path[->,black,thick] (n\x) edge[bend left=10] (n\y);     
                          \fi 
                          \ifnum \x>\y 
                          \path[->,black,thick] (n\x) edge[bend left=10] (n\y);     
                          \fi 
                        }
    \end{tikzpicture}
    &
    \begin{tikzpicture}[scale=1,transform shape]
      \draw[white] (-2,-1.7) -- (-2,2.3) -- (2,2.3) -- (2,-1.7) -- (-2,-1.7);
      \foreach \x [evaluate={\l=162-\x*72;}] in {1,...,5} 
               {
                 \node[draw,circle,thick,black,minimum size=0.7cm] (n\x) at ($(0,0)+(\l:1.2)$) {};
                 \node at (n\x) {\x};
               }

               \path[->,black,thick] (n1) edge[anchor=center,loop above] (n1);
               \path[->,black,thick] (n2) edge[anchor=center,loop above] (n2);
               \path[->,black,thick] (n3) edge[anchor=center,loop above] (n3);
               \path[->,black,thick] (n4) edge[anchor=center,loop above] (n4);
               \path[->,black,thick] (n5) edge[anchor=center,loop above] (n5);
    \end{tikzpicture}
    &
    \begin{tikzpicture}[scale=1, transform shape]
      \draw[white] (-2,-1.7) -- (-2,2.3) -- (2,2.3) -- (2,-1.7) -- (-2,-1.7);
      \foreach \x [evaluate={\l=162-\x*72;}] in {1,...,5} 
               {
                 \node[draw,circle,thick,black,minimum size=0.7cm] (n\x) at ($(0,0)+(\l:1.5)$) {};
                 \node at (n\x) {\x};
               }

               \foreach \x in {1,...,5} \foreach \y in {1,...,5}
                        {
                          \ifnum \x<\y 
                          \path[->,black,thick] (n\x) edge[bend left=10] (n\y);     
                          \fi 
                          \ifnum \x>\y 
                          \path[->,black,thick] (n\x) edge[bend left=10] (n\y);     
                          \fi 
                        }
    \end{tikzpicture}
    \\[-0.1cm]
    (a)&(b)&(c)
    \\[0.2cm]
    \begin{tikzpicture}[scale=1,transform shape]
      \draw[white] (-2,-1.7) -- (-2,2.3) -- (2,2.3) -- (2,-1.7) -- (-2,-1.7);
      \node[draw,circle,thick,black,minimum size=0.7cm] (n1) at (-0.65,0.5) {};
      \node[draw,circle,thick,black,minimum size=0.7cm] (n2) at (0.65,0.5) {};

      \node[black] at (n1) {1};
      \node[black] at (n2) {2};

      \path[->,black,thick] (n1) edge[anchor=center,loop above] (n1);
      \path[->,black,thick] (n1) edge (n2);
    \end{tikzpicture}
    &
    \begin{tikzpicture}[scale=1, transform shape]
      \draw[white] (-2,-1.7) -- (-2,2.3) -- (2,2.3) -- (2,-1.7) -- (-2,-1.7);
      \foreach \x [evaluate={\l=162-\x*72;}] in {1,...,5} 
               {
                 \node[draw,circle,thick,black,minimum size=0.7cm] (n\x) at ($(0,0)+(\l:1.2)$) {};
                 \node at (n\x) {\x};
               }

               \path[->,black,thick] (n1) edge[anchor=center,loop above] (n1);
               \path[->,black,thick] (n1) edge (n2);     
               \path[->,black,thick] (n1) edge (n3);     
               \path[->,black,thick] (n1) edge (n4);     
               \path[->,black,thick] (n1) edge (n5);     
    \end{tikzpicture}
    &
    \begin{tikzpicture}[scale=1, transform shape]
      \draw[white] (-2,-1.7) -- (-2,2.3) -- (2,2.3) -- (2,-1.7) -- (-2,-1.7);
      \foreach \x [evaluate={\l=162-\x*72;}] in {1,...,5} 
               {
                 \node[draw,circle,thick,black,minimum size=0.7cm] (n\x) at ($(0,0)+(\l:1.5)$) {};
                 \node at (n\x) {\x};
               }

               \path[->,black,thick] (n2) edge[anchor=center,loop right] (n2);
               \path[->,black,thick] (n3) edge[anchor=center,loop right] (n3);
               \path[->,black,thick] (n4) edge[anchor=center,loop left] (n4);
               \path[->,black,thick] (n5) edge[anchor=center,loop left] (n5);

               \foreach \x in {2,...,4} \foreach \y in {1,...,5}
                        {
                          \ifnum \x<\y 
                          \path[->,black,thick] (n\x) edge[bend left=10] (n\y);     
                          \fi 
                          \ifnum \x>\y 
                          \path[->,black,thick] (n\x) edge[bend left=10] (n\y);     
                          \fi 
                        }
            \path[->,black,thick] (n1) edge[bend left=10] (n2);     
            \path[->,black,thick] (n1) edge[bend left=10] (n3);     
            \path[->,black,thick] (n1) edge[bend left=10] (n4);     
            \path[->,black,thick] (n1) edge (n5);     
            \path[->,black,thick] (n5) edge[bend left=10] (n2);     
            \path[->,black,thick] (n5) edge[bend left=10] (n3);     
            \path[->,black,thick] (n5) edge[bend left=10] (n4);     

    \end{tikzpicture}
    \\[-0.1cm]
    (d)&(e)&(f)
\end{tabular}
\end{center}
\caption{Examples of feedback graphs: (a) \emph{full feedback}, (b) \emph{bandit feedback}, (c) \emph{loopless clique}, (d) \emph{apple tasting}, (e) \emph{revealing action}, (f) a clique minus a self-loop and another edge. }
\label{fig:graphs}
\end{figure}
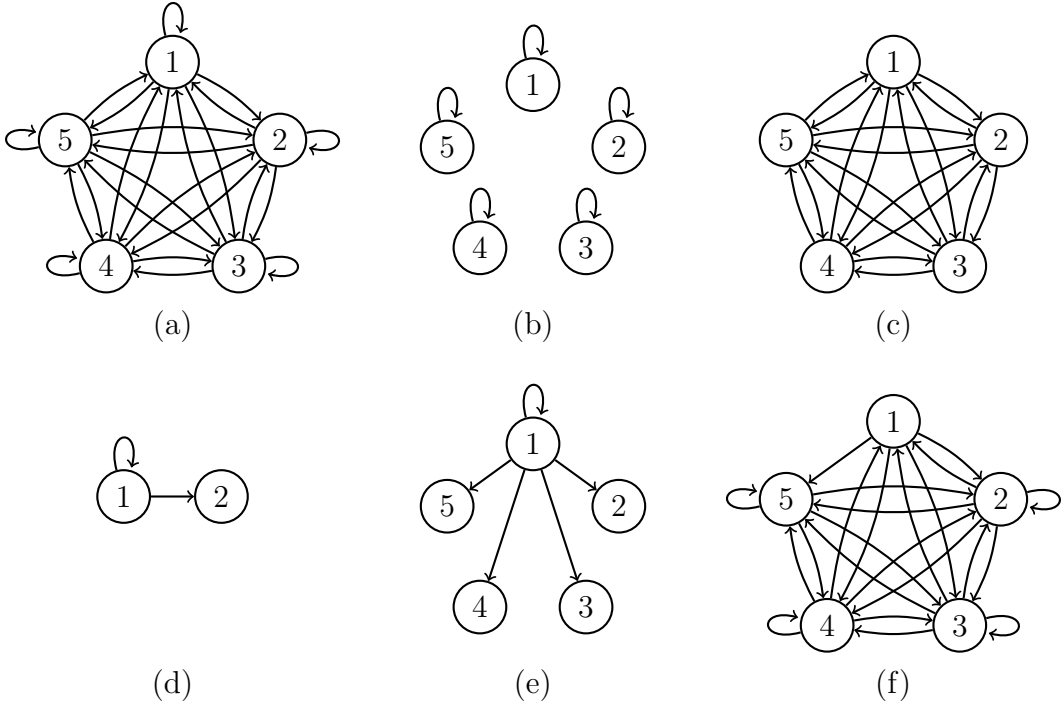

\section{Problem Setting and Main Results}
\label{sec:setting}
Let $G = (V,E)$ be a directed feedback graph over the set of actions
$V = \{1,\ldots,K\}$. For each $i\in V$, let $\nin(i) = \set{j\in V :
  (j,i) \in E}$ be the in-neighborhood of $i$ in $G$, and let
$\nout(i) = \set{j\in V : (i,j) \in E}$ be the out-neighborhood of $i$
in $G$.
If $i$ has a self-loop, that is $(i,i) \in E$, then $i \in \nin(i)$ and $i \in \nout(i)$.


Before the game begins, the environment privately selects a sequence
of loss functions $\ell_1,\ell_2.\dots$, where $\ell_t: V \mapsto
[0,1]$ for each $t \ge 1$. On each round $t=1,2,\dots$, the player
randomly chooses an action $I_t \in V$ and incurs the loss
$\ell_t(I_t)$. At the end of round $t$, the player receives the
feedback $\set{\big(j, \ell_{t}(j)\big) : j \in \nout(I_{t})}$.  In
words, the player observes the loss associated with each vertex in the
out-neighborhood of the chosen action $I_t$. In particular, if $I_t$
has no self-loop, then the player's loss $\ell_t(I_t)$ remains
unknown, and if the out-neighborhood of $I_t$ is empty, then the
player does not observe any feedback on that round.  The player's
\emph{expected regret} against a specific loss sequence
$\ell_1,\ldots,\ell_T$ is defined as $\E\big[ \sum_{t=1}^{T}
  \ell_{t}(I_{t}) \big] - \min_{i \in V} \sum_{t=1}^{T} \ell_{t}(i)$.  The
inherent difficulty of the $T$-round online learning problem induced
by the feedback graph $G$ is measured by the \emph{minimax regret},
denoted by $R(G,T)$ and defined as the minimum over all randomized
player strategies, of the maximum over all loss sequences, of the
player's expected regret.

\subsection{Main Results}
The main result of this paper is a complete characterization of the
minimax regret when the feedback graph $G$ is fixed and known to the
player.  Our characterization relies on various properties of $G$,
which we define below.
\begin{definition*}[Observability]
In a directed graph $G = (V,E)$ a vertex $i \in V$ is \emph{observable}
if $\nin(i) \neq \emptyset$. A vertex is \emph{strongly observable} if
either $\set{i} \subseteq \nin(i)$, or $V \setminus \set{i} \subseteq
\nin(i)$, or both.  A vertex is \emph{weakly observable} if it is
observable but not strongly.  A graph $G$ is observable if all its
vertices are observable and it is strongly observable if all its
vertices are strongly observable. A graph is weakly observable if it
is observable but not strongly.
\end{definition*}
In words, a vertex is observable if it has at least one incoming edge
(possibly a self-loop), and it is strongly observable if it has either
a self-loop or incoming edges from \emph{all} other vertices.  Note
that a graph with all of the self-loops is necessarily strongly
observable. However, a graph that is missing some of its self-loops
may or may not be observable or strongly observable.

\begin{definition*}[Weak Domination]
  In a directed graph $G = (V,E)$ with a set of weakly observable
  vertices $W \subseteq V$, a \emph{weakly dominating set} $D
  \subseteq V$ is a set of vertices that dominates $W$. Namely, for
  any $w \in W$ there exists $d \in D$ such that $w \in \nout(d)$. The
  \emph{weak domination number} of $G$, denoted by $\del(G)$, is the 
  size of the smallest weakly dominating set.
\end{definition*}
Our characterization also relies on a more standard graph-theoretic
quantity. An \emph{independent set} $S \subseteq V$ is a set of
vertices that are not connected by any edges. Namely, for any $u,v \in S$,  
$u \neq v$ it holds that $(u,v) \not \in E$.  
The \emph{independence number} $\alpha(G)$ of $G$ is the size of its 
largest independent set. Our characterization of the minimax regret 
rates is given by the following theorem.
\begin{theorem}
\label{th:main}
Let $G = (V,E)$ be a feedback graph with $\abs{V} \ge 2$, fixed and known in advance. 
Let $\alpha = \alpha(G)$ denote its independence number and let $\del = \del(G)$ denote its weak domination number. 
Then the minimax regret of the $T$-round online learning problem induced by $G$, where $T \ge |V|^3$, is
\begin{enumerate}[label=(\roman*),nosep]
\item $R(G,T) = \wt{\Theta}(\alpha^{1/2} \, T^{1/2})$ if $G$ is strongly observable;
\item $R(G,T) = \wt{\Theta}(\delta^{1/3} \, T^{2/3})$ if $G$ is weakly observable;
\item $R(G,T) = \Theta(T)$ if $G$ is not observable.
\end{enumerate}
\end{theorem}

As mentioned above, this characterization has some interesting
consequences. Any strongly observable graph can be turned into a
weakly observable graph by removing at most two edges. Doing so will
cause the minimax regret rate to jump from order $\sqrt{T}$ to order
$T^{2/3}$. Even more remarkably, removing these edges will cause the
minimax regret to switch from depending on the independence number to
depending on the weak domination number. A striking example of this
abrupt change is the \emph{loopy star} graph, which is the union of
the directed star (\cref{fig:graphs}e) and all of the self-loops
(\cref{fig:graphs}b). In other words, this example is a multi-armed
bandit problem with a revealing action. The independence number of
this graph is $K-1$, while its weak domination number is~$1$. Since
the loopy star is strongly observable, it induces a game with minimax
regret $\wt{\Theta}(\sqrt{TK})$. However, removing a single loop from
the feedback graph turns it into a weakly observable graph, and its
minimax regret rate changes to $\wt{\Theta}(T^{2/3})$ (with no
polynomial dependence on $K$).



\section{The \textsc{Exp3.G} Algorithm}
\label{s:algo}
The upper bounds for weakly and strongly observable graphs in \cref{th:main} are both achieved by an algorithm we introduce, called \textsc{Exp3.G} (see \cref{alg:alg}), which is a variant of the \textsc{Exp3-SET} algorithm for undirected feedback graphs \citep{Alonetal2013}.
\setlength{\algomargin}{0.75em}
\begin{algorithm}[t] 
\SetAlgoNoEnd
\SetAlgoNoLine
\SetArgSty{textrm}
\SetKwFor{For}{For}{}{}
\SetKwInput{KwParams}{Parameters}
\KwParams{Feedback graph $G = (V,E)$, learning rate $\eta>0$, \newline
exploration set $U \subseteq V$, exploration rate $\gamma \in [0,1]$}
\BlankLine

Let $u$ be the uniform distribution over $U$\;
Initialize $q_1$ to the uniform distribution over $V$\;
\For{round $t=1,2,\dots$}
{
	Compute $p_{t} = (1-\gamma) q_{t} + \gamma u$\;
	Draw $I_{t} \sim p_{t}$, play $I_{t}$ and incur loss $\ell_t(I_t)$\;
	Observe $\set{(i,\ell_{t}(i)) : i \in \nout(I_{t})}$\;
	Update
	\begin{align} \label{eq:estimator}
		\forall ~ i \in V \qquad& \ellh_{t}(i) \eq \frac{\ell_{t}(i)}{P_{t}(i)}\,\ind{i \in \nout(I_{t})} , 
		\qquad\text{with}\qquad
		P_{t}(i) \eq \sum_{\mathclap{j \in \nin(i)}} p_{t}(j) \text{ \;}&\\ 
		\forall ~ i \in V \qquad& q_{t+1}(i) \eq \frac{ q_{t}(i) \exp(-\eta \ellh_{t}(i)) }
			{ \sum_{j \in V} q_{t}(j) \exp(-\eta \ellh_{t}(j)) }
			\text{ \;} \label{eq:update}&
	\end{align}
}
\vspace{-0.75em}
\caption{%
\textsc{Exp3.G}: online learning with a feedback graph
}
\label{alg:alg}
\end{algorithm}

Similarly to \textsc{Exp3} and \textsc{Exp3.SET}, our algorithm uses importance sampling to construct unbiased loss estimates with controlled variance. 
Indeed, notice that $P_{t}(i) = \Pr(i \in \nout(I_{t}))$ is simply the probability of observing the loss $\ell_{t}(i)$ upon playing $I_{t} \sim p_{t}$.
Hence, $\ellh_{t}(i)$ is an unbiased estimate of the true loss $\ell_{t}(i)$, and for all $t$ and $i \in V$ we have
\begin{align}
\label{eq:unbiased}
	\E_{t}[\ellh_{t}(i)]
	\eq \ell_{t}(i)
	\quad\text{and}\quad
	\E_{t}[\ellh_{t}(i)^{2}]
	\eq \frac{\ell_{t}(i)^{2}}{P_{t}(i)}
	~.
\end{align}
The purpose of the exploration distribution $u$ is to control the variance of the loss estimates by providing a lower bound on $P_t(i)$ for those $i \in V$ in the support of $u$; this ingredient will turn out to be essential to our analysis.

We now state the upper bounds on the regret achieved by \cref{alg:alg}. 
\begin{theorem}
\label{th:mainupper}
Let $G=(V,E)$ be a feedback graph with $K = \abs{V}$, independence number $\alpha = \alpha(G)$ and weakly dominating number $\del = \del(G)$. 
Let $D$ be a weakly dominating set such that $|D| = \del$.
The expected regret of \cref{alg:alg} on the online learning problem induced by $G$ satisfies the following:
\begin{enumerate}[label=(\roman*),nosep]
\item
if $G$ is strongly observable, then for $U=V$, 
$
    \gamma 
= 
	\min\bigl\{ \bigl(\tfrac{1}{\alpha T}\bigr)^{1/2} , \tfrac{1}{2} \bigr\}
$
and
$
\eta =  2\gamma ,
$
the expected regret against any loss sequence is $\O( \alpha^{1/2} T^{1/2} \ln(KT) )$;
\item
if $G$ is weakly observable and $T \ge K^3 \ln(K) / \del^2$, then for $U=D$,
$
    \gamma = \min\!\big\{ \big(\tfrac{\delta\ln K}{T}\big)^{1/3}, \tfrac{1}{2} \big\}
$
and
$
    \eta = \tfrac{\gamma^{2}}{\delta} ,
$
the~expected regret against any loss sequence is 
$
	\O\bigl( (\delta\ln K)^{1/3}T^{2/3} \bigr) .
$
\end{enumerate}
\end{theorem}
In the previously studied self-aware case (i.e., strongly observable with self-loops), our result matches the bounds of \cite{AlonCGMMS14,kocak2014efficient}.
The tightness of our bounds in all cases is discussed in \cref{s:lower} below.

\subsection{A Tight Bound for the Loopless Clique}
One of the simplest examples of a feedback graph that is not
self-aware is the loopless clique (\cref{fig:graphs}c). This graph is
strongly observable with an independence number of $1$, so
\cref{th:mainupper} guarantees that the regret of
\cref{alg:alg} in the induced game is $\O(\sqrt{T} \ln(KT))$.
However, in this case we can do better than \cref{th:mainupper} and prove (see \cref{s:clique}) that the regret of the same algorithm is actually $\O(\sqrt{T\ln{K}})$, which is the same as the regret rate of the full feedback game (\cref{fig:graphs}a). In other words, if we start with full feedback and then hide the player's own loss, the regret rate remains the same (up to constants).
\begin{theorem}
\label{th:blind}
For any sequence of loss functions $\ell_{1},\ldots,\ell_{T}$, where
$\ell_t:V \mapsto [0, 1]$, the regret of \cref{alg:alg}, with the loopless clique feedback graph and with parameters $\eta = \sqrt{(\ln K)/(2T)}$ and $\gamma=2\eta$, is upper-bounded by $5\sqrt{T\ln K}$. 
\end{theorem}


\subsection{Refined Second-order Bound for Hedge}
%
Our analysis of \textsc{Exp3.G} builds on a new second-order regret bound for the classic Hedge algorithm.%
\footnote{A second-order regret bound controls the regret with an expression that depends on a quantity akin to the second moment of the losses.} 
Recall that Hedge \citep{FS97} operates in the full feedback setting (see \cref{fig:graphs}a), where at time $t$ the player has access to losses $\ell_s(i)$ for all $s < t$ and $i \in V$. Hedge draws action $I_t$ from the distribution $p_t$ defined by
\begin{equation}
\label{eqn:mw}
	\forall ~ i \in V ~, \qquad 
    q_t(i) \eq \frac{\exp\big(-\eta \sum_{s=1}^{t-1} \ell_s(i)\big)}
    	{\sum_{j \in V}\exp\big(-\eta \sum_{s=1}^{t-1} \ell_s(j)\big)} 
    ~,
\end{equation}
where $\eta$ is a positive learning rate. 
The following novel regret bound is key to proving that our algorithm achieves tight bounds over the regret (to within logarithmic factors).
\begin{lemma}
\label{lem:regret-mw2}
Let $q_{1},\ldots,q_{T}$ be the probability vectors defined by \cref{eqn:mw} for a sequence of loss functions $\ell_{1},\ldots,\ell_{T}$ such that~$\ell_{t}(i) \ge 0$ for all $t=1,\dots,T$ and $i \in V$.
For each $t$, let $S_{t}$ be a subset of $V$ such that $\ell_{t}(i) \le 1/\eta$ for all $i \in S_{t}$.
Then, for any $i^\st \in V$ it holds that
\begin{align*}
	\sum_{t=1}^{T} \sum_{i \in V} q_{t}(i) \ell_{t}(i)
	- \sum_{t=1}^{T} \ell_{t}(i^\st) 
\leq
	\frac{\ln K}{\eta}
	+ \eta \sum_{t=1}^{T} \!\left( 
	\sum_{i \in S_{t}} q_{t}(i) \bigl(1-q_{t}(i)\bigr) \ell_{t}(i)^{2} 
	+ \sum_{i \notin S_{t}} q_{t}(i) \ell_{t}(i)^{2} 
	\!\right)
	\!.
\end{align*}
\end{lemma}
See \cref{sec:proofs} for a proof of this result. 
The standard second-order regret bound of Hedge \citep[see, e.g.,][]{CBMS07} is obtained by setting $S_{t} = \emptyset$ for all $t$. Therefore, our bound features a slightly improved dependence (i.e., the $1-q_{t}(i)$ factors) on actions whose losses do not exceed $1/\eta$.
Indeed, in the analysis of \textsc{Exp3.G}, we apply the above lemma to the loss estimates $\ellh_t(i)$, and include in the sets $S_t$ all strongly observable vertices $i$ that do not have a self-loop. This allows us to gain a finer control on the variances $\ell_t(i)^2\big/P_t(i)$ of such vertices.

\subsection{Proof of \cref{th:mainupper}}

We now turn to prove \cref{th:mainupper}.
For the proof, we need the following graph-theoretic result, which is a variant of \citet[Lemma 16]{AlonCGMMS14}; for completeness, we include a proof in \cref{sec:proofs}.
\begin{lemma} \label{lem:alon}
Let $G = (V,E)$ be a directed graph with $\abs{V}=K$, in which each node $i \in V$ is assigned a positive weight $w_{i}$.
Assume that $\sum_{i \in V} w_{i} \le 1$, and that $w_{i} \ge \eps$ for all $i \in V$ for some constant $0<\eps<\half$.
Then
\begin{align*}
	\sum_{i \in V} \frac{w_{i}}{w_{i}+\sum_{j \in \nin(i)} w_{j}}
	\leq 4 \alpha \ln \frac{4K}{\alpha\eps}
\end{align*}
where $\alpha = \alpha(G)$ is the independence number of $G$.
\end{lemma}
\begin{proof}[Proof of \cref{th:mainupper}]
Without loss of generality, we may assume that $K \ge 2$.
The proof proceeds by applying \cref{lem:regret-mw2} and upper bounding the second-order terms it introduces.
Indeed, since the distributions $q_1,q_2,\dots$ generated by \cref{alg:alg} via \cref{eq:update} are of the form given by \cref{eqn:mw}, with the losses $\ell_t$ replaced by the nonnegative loss estimates $\ellh_t$, we may apply \cref{lem:regret-mw2} to these distributions and loss estimates. 
The way we apply the lemma differs between the strongly observable and weakly observable cases, and we treat each separately.

First, assume that $G$ is strongly observable, implying that the exploration distribution $u$ is uniform on $V$.
Notice that for any $i \in V$ without a self-loop, namely with $i \notin \nin(i)$, we have $j \in \nin(i)$ for all $j \neq i$, and so $P_{t}(i) = 1-p_{t}(i)$. On the other hand, by the definition of $p_{t}$ and since $\eta = 2\gamma$ and $K \ge 2$, we have
$
	p_{t}(i) 
	= (1-\gamma) q_{t}(i) + \frac{\gamma}{K} 
	\le 1-\gamma + \tfrac{\gamma}{2}
	= 1-\eta
$,
so that $P_{t}(i) \ge \eta$. Thus, we can apply \cref{lem:regret-mw2} with $S_{t} = S = \set{i \,:\, i \notin \nin(i)}$ to the vectors $\wh{\ell}_{1},\ldots,\wh{\ell}_{T}$ and take expectations, and obtain that
\begin{align*}
    &\EE{ \sum_{t=1}^{T} \sum_{i \in V} q_{t}(i) 
    	\E_{t}[\wh{\ell}_{t}(i)]
    - \sum_{t=1}^{T} \E_{t}[\wh{\ell}_{t}(i^{\st})] }
\leq
    \frac{\ln K}{\eta}
\\
	&\qquad\qquad
    + \eta \, \sum_{t=1}^{T} \EE{ 
	\sum_{i \in S} q_{t}(i) (1-q_{t}(i)) \E_{t}[\wh{\ell}_{t}(i)^{2}]
	+ \sum_{i \notin S} q_{t}(i) \E_{t}[\wh{\ell}_{t}(i)^{2}] }
\end{align*}
for any fixed $i^{\st} \in V$.
Recalling \cref{eq:unbiased} and $P_t(i) = 1 - p_t(i)$ for all $i \in S$, we get
\begin{align*}
	\EE{ \sum_{t=1}^{T} \sum_{i \in V} q_{t}(i) \ell_{t}(i) }
	- \sum_{t=1}^{T} \ell_{t}(i^{\st}) 
\leq 
	\frac{\ln K}{\eta}
	+ \eta \, \sum_{t=1}^{T} \EE{
		\sum_{i \in S} q_{t}(i) \frac{1-q_{t}(i)}{1-p_{t}(i)} 
	+ \sum_{i \notin S} \frac{q_{t}(i)}{P_{t}(i)} }
~.
\end{align*}
The sum over $i \in S$ on the right-hand side is bounded as follows:
\begin{align*}
	\sum_{t=1}^{T} \sum_{i \in S} q_{t}(i) \frac{1-q_{t}(i)}{1-p_{t}(i)}
	\leq 2 \sum_{t=1}^{T} \sum_{i \in S} q_{t}(i)
	\leq 2T
	~.
\end{align*}
For the second sum, recall that any $i \notin S$ has a self-loop in the feedback graph, and also that $p_{t}(i) \ge \frac{\gamma}{K}$ as a result of mixing in the uniform distribution over $V$. 
Hence, we can use $p_{t}(i) \ge (1-\gamma) q_{t}(i) \ge \frac{1}{2} q_{t}(i)$ and apply \cref{lem:alon} with $\eps = \frac{\gamma}{K}$ that yields
\begin{align*}
	\sum_{i \notin S} \frac{q_{t}(i)}{P_{t}(i)}	
	\leq 2 \sum_{i \notin S} \frac{p_{t}(i)}{P_{t}(i)}
	\leq 8\alpha \ln \frac{K^{2}}{4\gamma}
	~.
\end{align*}
Putting everything together, and using the fact that $p_{t}(i) \le q_{t}(i) + \gamma u(i)$ to obtain
\begin{align} \label{eq:main2b}
	\sum_{i \in V} p_{t}(i) \ell_{t}(i)
	\leq \sum_{i \in V} q_{t}(i) \ell_{t}(i) + \gamma
	~,
\end{align}
results with the regret bound
\begin{align*}
	\EE{ \sum_{t=1}^{T} \sum_{i \in V} p_{t}(i) \ell_{t}(i) }
   	- \sum_{t=1}^{T} \ell_{t}(i^{\st}) 
\leq 
	\gamma T
	+ \frac{\ln K}{\eta} 
	+ 2\eta T \lr{ 1 + 4\alpha \ln\frac{K^{2}}{4\gamma} } 
	~.
\end{align*}
Substituting the chosen values of $\eta$ and $\gamma$ gives the first claim of the theorem.

Next, assume that $G$ is only weakly observable. Let $D \subseteq V$ be a weakly dominating set supporting the exploration distribution $u$, with $\abs{D} = \del$. 
Similarly to the strongly observable case, we apply \cref{lem:regret-mw2} to the vectors $\ellh_{1},\ldots,\ellh_{T}$, but in this case we set $S_t = \emptyset$ for all $t$.
%
%
Using \cref{eq:unbiased,eq:main2b} and proceeding exactly as the strongly observable case, we obtain
\begin{align*}
    \EE{ \sum_{t=1}^{T} \sum_{i \in V} p_{t}(i) \ell_{t}(i) } 
    - \sum_{t=1}^{T} \ell_{t}(i^{\st})
\leq
    \gamma T + \frac{\ln K}{\eta}
	+ \eta \, \sum_{t=1}^{T} \EE{ \sum_{i \in V} \frac{q_{t}(i)}{P_{t}(i)} }
\end{align*}
for any fixed $i^{\st} \in V$. 
In order to bound the expectation in the right-hand side, consider again the set $S = \set{i \,:\, i \notin \nin(i)}$ of vertices without a self-loop, and observe that  $P_{t}(i) = \smash{\sum_{j \in \nin(i)}} p_{t}(j) \ge \frac{\gamma}{\delta}$ for all $i \in S$. Indeed, if $i$ is weakly observable then there exists some $k \in D$ such that $k \in \nin(i)$ and $p_t(k) \ge \frac{\gamma}{\delta}$ because the exploration distribution $u$ is uniform over $D$; if $i$ is strongly observable then the same holds since $i$ does not have a self-loop and thus must be dominated by all other vertices in the graph.
Hence,
\begin{align*}
    \sum_{i \in V} \frac{q_{t}(i)}{P_{t}(i)}
\eq
    \sum_{i \in S} \frac{q_{t}(i)}{P_{t}(i)}
    + \sum_{i \notin S} \frac{q_{t}(i)}{P_{t}(i)}
\leq
    \frac{\delta}{\gamma} + 2K
~,
\end{align*}
where we used $P_t(i) \ge p_t(i) \ge (1-\gamma)q_t(i) \ge \frac{1}{2}q_t(i)$ to bound the sum over the vertices having a self-loop. 
Therefore, we may write
\begin{align*}
    \EE{ \sum_{t=1}^{T} \sum_{i \in V} p_{t}(i) \ell_{t}(i) } 
    - \sum_{t=1}^{T} \ell_{t}(i^{\st})
\leq
    \gamma T + \frac{\ln K}{\eta}
    + \frac{\eta\delta}{\gamma}T + 2\eta K T~.
\end{align*}
Substituting our choices of $\eta$ and $\gamma$, we obtain the second claim of the theorem.
\end{proof}


\section{Lower Bounds}
\label{s:lower}

In this section we prove lower bounds on the minimax regret for non-observable and weakly observable graphs. Together with \cref{{th:mainupper}} and the known lower bound of $\Omega\bigl(\sqrt{\alpha(G)T}\bigr)$ for strongly observable graphs \citep[Theorem~5]{AlonCGMMS14},%
\footnote{While \cite{AlonCGMMS14} only consider the special case of graphs that have self-loops at all vertices, their lower bound applies to any strongly observable graph: we can simply add any missing self-loops to the graph, without changing its independence number $\alpha$. The resulting learning problem, whose minimax regret is $\Omega\bigl(\sqrt{\alpha T}\bigr)$, is only easier for the player who may ignore the additional feedback.}
these results complete the proof of \cref{th:main}.
We remark that their lower bound applies when $T \ge \alpha(G)^3$, which includes our regime of interest.
We begin with a simple lower bound for non-observable feedback graphs.
\begin{theorem}
\label{th:notobser}
If $G=(V,E)$ is not observable and $\abs{V} \ge 2$, then for any player algorithm there exists a sequence of loss functions $\ell_{1},\ell_2,\ldots : V \mapsto [0,1]$ such that the player's expected regret is at least~$\tfrac{1}{4} T$.
\end{theorem}
The proof is straightforward: if $G$ is not observable, then it is possible to find a vertex of $G$ with no incoming edges; the environment can then set the loss of this vertex to be either $0$ or $1$ on all rounds of the game, and the player has no way of knowing which is the case.
For the formal proof, refer to \cref{sec:lower-proofs}.
Next, we prove a lower bound for weakly observable feedback graphs.
\begin{theorem}
\label{th:lower-weak}
If $G=(V,E)$ is weakly observable with $K = \abs{V} \ge 2$ and weak domination number $\del = \del(G)$, then for any randomized player algorithm and for any time horizon $T$ there exists a sequence of loss functions $\ell_{1},\ldots,\ell_{T} : V \mapsto [0,1]$ such that the player's expected regret is at least 
$
	\tfrac{1}{150} \bigl(\del/\ln^2\!{K}\bigr)^{1/3} T^{2/3} .
$
\end{theorem}
The proof relies on the following graph-theoretic result, relating the notions of domination and independence in directed graphs.
\begin{lemma} \label{lem:noga}
Let $G = (V,E)$ be a directed graph over $\abs{V} = n$ vertices, and let $W \subseteq V$ be a set of vertices whose minimal dominating set is of size $k$. 
Then, $W$ contains an independent set $U$ of size at least $\tfrac{1}{50} k/\ln{n}$, with the property that any vertex of $G$ dominates at most $\ln{n}$ vertices of~$U$.
\end{lemma}
\begin{proof}
If $k < 50 \ln n$ the statement is vacuous; hence, in what follows we assume $k \ge 50\ln n$.
Let $\beta = (2\ln{n})/k < 1$.
Our first step is to prove that $W$ contains a non-empty set $R$ such that each vertex of $G$ dominates at most $\beta$ fraction of $R$, namely such that
$\abs{\nout(v) \cap R} \le \beta \abs{R}$ for all $v \in V$.
%
To prove this, consider the following iterative process: initialize $R = W$, and as long as there exists a vertex $v \in V$ such that $\abs{\nout(v) \cap R} > \beta \abs{R}$, remove all the vertices $v$ dominates from $R$.
Notice that the process cannot continue for $k$ (or more) iterations, since each step the size of $R$ decreases at least by a factor of $1-\beta$, so after $k-1$ steps we have 
$
\abs{R} \le n (1-\beta)^{k-1} < n e^{-\beta k/2} = 1 .
$
On the other hand, the process cannot end with $R = \emptyset$, as in that case the vertices~$v$ found along the way form a dominating set of $W$ whose size is less than~$k$, which is a contradiction to our assumption.
Hence, the set $R$ at the end of process must be non-empty and satisfy $\abs{\nout(v) \cap R} \le \beta \abs{R}$ for all $v \in V$, as claimed.

Next, consider a random set $S \subseteq R$ formed by picking a multiset $\wt{S}$ of $m = \floor{\tfrac{1}{10\beta}}$ elements from $R$ independently and uniformly at random (with replacement), and discarding any repeating elements.
Notice that $m \le \tfrac{1}{10} \abs{R}$, as $\abs{R} \ge \tfrac{1}{\beta} \abs{\nout(v) \cap R}$ for any $v \in V$, and for some $v$ the right-hand side is non-zero.
The proof proceeds via the probabilistic method: we will show that with positive probability, $S$ contains an independence set as required, which would give the theorem.

We first observe the following properties of the set $S$. 
\begin{claim*}
With probability at least $\frac{3}{4}$, it holds that $\abs{S} \ge \frac{1}{10}m$.
\end{claim*}
To see this, note that each element from $R$ is not included in $\wt{S}$ with probability $(1-\tfrac{1}{r})^{m} \le e^{-m/r}$ with $r=\abs{R}$.
Since $m \le \tfrac{1}{10}r$, the expected size of $S$ is at least 
$
r(1-e^{-m/r}) = r e^{-m/r} ( e^{m/r} - 1 ) \ge m e^{-m/r} \ge \tfrac{9}{10}m,
$ 
where both inequality use $e^{x} \ge x+1$.
Since always $\abs{S} \le m$, Markov's inequality shows that $\abs{S} \ge \tfrac{1}{10}m$ with probability at least $\tfrac{3}{4}$; otherwise, we would have $\E\bigl[\abs{S}\bigr] \le \tfrac{1}{10}m + m \Pr\bigl(\abs{S} \ge \tfrac{1}{10}m\bigr) < \tfrac{9}{10}m$.
\begin{claim*}
With probability at least $\frac{3}{4}$, we have $\abs{\nout(v) \cap S} \le \ln{n}$ for all $v \in V$.
\end{claim*}
Indeed, fix some $v \in V$ and recall that $v$ dominates at most a $\beta$ fraction of the vertices in $R$, so each element of $\wt{S}$ (that was chosen uniformly at random from $R$) is dominated by $v$ with probability at most $\beta$. 
Hence, the random variable $\wt{X}_{v} = \abs{\nout(v) \cap \wt{S}}$ has a binomial distribution $\mathsf{Bin}(m,p)$ with $p \le \beta$. 
By a standard binomial tail bound,
$$
\Pr(\wt{X}_{v} \ge \ln{n}) 
\leq
{m \choose \ln{n}} \, \beta^{\ln{n}}
\leq
(m\beta)^{\ln{n}}
\leq
e^{-2\ln{n}}
\eq
\frac{1}{n^{2}}
~.
$$
The same bound holds also for the random variable $X_{v} = \abs{\nout(v) \cap S}$, that can only be smaller than $\wt{X}_{v}$.
Our claim now follows from a union bound over all $v \in V$.
\begin{claim*}
With probability at least $\frac{3}{4}$, we have $\frac{1}{\abs{S}} \sum_{v \in S} \abs{\nout(v) \cap S} \le \tfrac{1}{2}$.
\end{claim*}
To obtain this, we note that for each $v \in V$ the random variable $X_{v} = \abs{\nout(v) \cap S}$ defined above has $\E[X_{v}] \le \E[\wt{X}_{v}] \le m\beta \le \tfrac{1}{10}$, and therefore $\E\bigl[ \tfrac{1}{\abs{S}} \sum_{v \in S} X_{v} \bigr] \le \tfrac{1}{10}$.
By Markov's inequality we then have $\tfrac{1}{\abs{S}} \sum_{v \in S} X_{v} > \tfrac{1}{2}$ with probability less than $\tfrac{1}{5}$, which gives the claim. 

\medskip
The three claims together imply that there exists a set $S \subseteq W$ of size at least $\frac{1}{10}m$, such that any $v \in V$ dominates at most $\ln{n}$ vertices of $S$, and the average degree of the induced undirected graph over $S$ is at most $1$.
Hence, by Tur\'{a}n's Theorem,%
\footnote{Tur\'{a}n's Theorem \citep[e.g.,][]{alon2011probabilistic} states that in any undirected graph whose average degree is $d$, there is an independent set of size $n/(d+1)$.}
$S$ contains an independent set $U$ of size $\tfrac{1}{20} m \ge \tfrac{1}{50} k/\ln{n}$.
This concludes the proof, as each $v \in V$ dominates at most $\ln{n}$ vertices of~$U$. 
\end{proof}

Given \cref{lem:noga}, the idea of the proof is quite intuitive; here we only give a sketch of the proof, and defer the formal details to \cref{sec:lower-proofs}.
%
\begin{proof}[Proof of \cref{th:lower-weak} (sketch)]
First, we use the lemma to find an independent set $U$ of weakly observable vertices of size $\wt\Omega(\del)$, with the crucial property that each vertex in the entire graph dominates at most $\wt\O(1)$ vertices of $U$.
Then, we embed in the set $U$ a hard instance of the stochastic multiarmed bandit problem, in which the optimal action has expected loss smaller by $\eps$ than the expected loss of the other actions in $U$. To all other vertices of the graph, we assign the maximal loss of $1$.
Hence, unless the player is able to detect the optimal action, his regret cannot be better than $\Omega(\eps T)$.

The main observation is that, due to the properties of the set $U$, in order to obtain accurate estimates of the losses of all actions in $U$ the player has to use $\wt\Omega(\del)$ different actions outside of $U$ and pick each for $\Omega(1/\eps^2)$ times. Since each such action entails a constant instantaneous regret, the player has to pay an $\Omega(\del/\eps^2)$ penalty in his cumulative regret for exploration.
The overall regret is thus of order $\Omega\lrbig{\! \min\{\eps T, \del/\eps^2\} \!}$, which is maximized at $\eps=(\del/T)^{1/3}$ and gives the stated lower bound.
\end{proof}



\section{Time-Varying Feedback Graphs}
\label{s:exte}
The setting discussed above can be generalized by allowing the
feedback graphs to change arbitrarily from round to round (see
\citet{MS11,Alonetal2013,kocak2014efficient}). Namely, the environment
chooses a sequence of feedback graphs $G_1,\ldots,G_T$ along with the
sequence of loss functions. We consider two different variants of this
setting: in the \emph{informed} model, the player observes $G_t$ at
the beginning of round $t$, before drawing the action $I_t$. In the
harder \emph{uninformed} model, the player observes $G_t$ at the end
of round $t$, after drawing $I_t$. In this section, we discuss how our
algorithm can be modified to handle time-varying feedback graphs, and
whether this generalization increases the minimax regret of the
induced online learning problem.

\paragraph{Strongly Observable.} 

If $G_1,\ldots,G_T$ are all strongly observable, \cref{alg:alg} and
its analysis can be adapted to the time-varying setting (both informed
and uninformed) with only a few cosmetic modifications. Specifically,
we replace $G$ with $G_t$, to define time-dependent neighborhoods,
$\nout_t$ and $\nin_t$, in \cref{eq:estimator} of the algorithm. This
modification holds in both the informed and uninformed models because
the structure of the feedback graph is only used to update $q_{t+1}$,
which takes place after the action $I_t$ is chosen. Moreover, the
upper-bound in \cref{th:mainupper} can be adapted to the time-varying
model by replacing $\alpha$ with $\frac{1}{T} \sum_{t=1}^T \alpha_t$,
where each $\alpha_t$ is the independence number of the corresponding
$G_t$ (e.g., using a doubling trick, or an adaptive learning rate as in \citet{kocak2014efficient}).

\paragraph{Weakly Observable, Informed.}
If $G_1,\ldots,G_T$ are all weakly observable, \cref{alg:alg} can
again be adapted to the informed time-varying model, but the required
modification is more substantial than before, and in particular, relies on the fact that $G_t$ is known before the prediction on round $t$ is made.
The exploration set $U$
must change from round to round, according to the feedback
graph. Specifically, we choose the exploration set on round $t$ to be
$D_t$, the smallest weakly dominating set in $G_t$. We then define
$u_t$ to be the uniform distribution over this set, and $p_t =
(1-\gamma)q_t + \gamma u_t$. 
Again, via standard techniques, the upper-bound in \cref{th:mainupper} can be adapted to this setting by replacing $\delta$ with $\frac{1}{T} \sum_{t=1}^T
\delta_t$, where $\delta_t = |D_t|$.

\paragraph{Weakly Observable, Uninformed.}

So far, we discussed cases where the minimax regret rates of our
problem do not increase when we allow the feedback graphs to change
from round to round. However, if $G_1,\ldots,G_T$ are all
weakly observable and they are revealed according to the uninformed
model, then the minimax regret can strictly increase. Recall that
\cref{th:main} states that the minimax regret for a constant
weakly observable graph is $\wt{\Theta}(\delta^{1/3} \, T^{2/3})$,
where $\delta$ is the size of the smallest weakly dominating set. We
now show that the minimax regret in the analogous uninformed setting
is $\wt{\Theta}(K^{1/3} \, T^{2/3})$, where $K$ is the number of
actions. The $\wt{\O}(K^{1/3} \, T^{2/3})$ upper bound is obtained by
running \cref{alg:alg} with uniform exploration over the entire set of
actions (namely, $U = V$). To show that this bound is tight, we state
the following matching lower bound.
\begin{theorem} \label{th:weak-sep}
For any randomized player strategy in the uninformed feedback model, there exists a sequence of weakly observable graphs $G_{1},\ldots,G_{T}$ over a set $V$ of $K \ge 4$ actions with $\del(G_t)=\alpha(G_t)=1$ for all $t$, and a sequence of loss functions $\ell_{1},\ldots,\ell_{T} : V \mapsto [0,1]$, such that the player's expected regret is at least $\tfrac{1}{16} K^{1/3} T^{2/3}$.
\end{theorem}

We sketch the proof below, and present it in full detail in \cref{sec:lower-proofs}.

\begin{proof}[Proof (sketch)]
For each $t=1,\ldots,T$, construct the graph $G_{t}$ as follows: start
with the complete graph over $K$ vertices (that includes all
self-loops), and then remove the self-loop and all edges incoming to
$i=1$ except of a single edge incoming from some vertex $j_{t} \ne 1$
chosen arbitrarily.  Notice that the resulting graph is
weakly observable (each vertex is observable, but $i=1$ is only
weakly observable), has $\del(G_t)=1$ since $j_{t}$ dominates the entire graph, and $\alpha(G_t)=1$ as each two vertices are connected by at least one edge.  However, for
observing the loss of $i=1$ the player has to ``guess''  the revealing action $j_{t}$, that might change arbitrarily from round to round. 
This random guessing of one out of $\Omega(K)$ actions introduces the $K^{1/3}$ factor in the resulting bound.
\end{proof}

\subsection*{Acknowledgements}
We thank S\'ebastien Bubeck for helpful discussions during various stages of this work, and G\'abor Bart\'ok for clarifying the connections to observability in partial monitoring.

\bibliographystyle{abbrvnat}
\bibliography{clique,nicolo}

\begin{thebibliography}{16}
\providecommand{\natexlab}[1]{#1}
\providecommand{\url}[1]{\texttt{#1}}
\expandafter\ifx\csname urlstyle\endcsname\relax
  \providecommand{\doi}[1]{doi: #1}\else
  \providecommand{\doi}{doi: \begingroup \urlstyle{rm}\Url}\fi

\bibitem[Alon and Spencer(2008)]{alon2011probabilistic}
N.~Alon and J.~H. Spencer.
\newblock \emph{The Probabilistic Method}.
\newblock John Wiley \& Sons, 2008.

\bibitem[Alon et~al.(2013)Alon, Cesa-Bianchi, Gentile, and
  Mansour]{Alonetal2013}
N.~Alon, N.~Cesa-Bianchi, C.~Gentile, and Y.~Mansour.
\newblock From bandits to experts: A tale of domination and independence.
\newblock In \emph{Advances in Neural Information Processing Systems 26}, pages
  1610--1618. Curran Associates, Inc., 2013.

\bibitem[Alon et~al.(2014)Alon, Cesa{-}Bianchi, Gentile, Mannor, Mansour, and
  Shamir]{AlonCGMMS14}
N.~Alon, N.~Cesa{-}Bianchi, C.~Gentile, S.~Mannor, Y.~Mansour, and O.~Shamir.
\newblock Nonstochastic multi-armed bandits with graph-structured feedback.
\newblock \emph{CoRR}, abs/1409.8428, 2014.

\bibitem[Antos et~al.(2013)Antos, Bart{\'o}k, P{\'a}l, and
  Szepesv{\'a}ri]{antos2013toward}
A.~Antos, G.~Bart{\'o}k, D.~P{\'a}l, and C.~Szepesv{\'a}ri.
\newblock Toward a classification of finite partial-monitoring games.
\newblock \emph{Theoretical Computer Science}, 473:\penalty0 77--99, 2013.

\bibitem[Auer et~al.(2002)Auer, Cesa-Bianchi, Freund, and
  Schapire]{AuerCeFrSc02}
P.~Auer, N.~Cesa-Bianchi, Y.~Freund, and R.~E. Schapire.
\newblock The nonstochastic multiarmed bandit problem.
\newblock \emph{SIAM Journal on Computing}, 32\penalty0 (1):\penalty0 48--77,
  2002.

\bibitem[Bart{\'o}k et~al.(2014)Bart{\'o}k, Foster, P{\'a}l, Rakhlin, and
  Szepesv{\'a}ri]{bartok2014partial}
G.~Bart{\'o}k, D.~P. Foster, D.~P{\'a}l, A.~Rakhlin, and C.~Szepesv{\'a}ri.
\newblock Partial monitoring---classification, regret bounds, and algorithms.
\newblock \emph{Mathematics of Operations Research}, 39\penalty0 (4):\penalty0
  967--997, 2014.

\bibitem[Cesa-Bianchi and Lugosi(2006)]{cbl06}
N.~Cesa-Bianchi and G.~Lugosi.
\newblock \emph{Prediction, learning, and games}.
\newblock Cambridge University Press, 2006.

\bibitem[Cesa-Bianchi et~al.(1997)Cesa-Bianchi, Freund, Haussler, Helmbold,
  Schapire, and Warmuth]{cb+97}
N.~Cesa-Bianchi, Y.~Freund, D.~Haussler, D.~Helmbold, R.~Schapire, and
  M.~Warmuth.
\newblock How to use expert advice.
\newblock \emph{Journal of the ACM}, 44\penalty0 (3):\penalty0 427--485, 1997.

\bibitem[Cesa{-}Bianchi et~al.(2007)Cesa{-}Bianchi, Mansour, and
  Stoltz]{CBMS07}
N.~Cesa{-}Bianchi, Y.~Mansour, and G.~Stoltz.
\newblock Improved second-order bounds for prediction with expert advice.
\newblock \emph{Machine Learning}, 66\penalty0 (2-3):\penalty0 321--352, 2007.

\bibitem[Freund and Schapire(1997)]{FS97}
Y.~Freund and R.~Schapire.
\newblock A decision-theoretic generalization of on-line learning and an
  application to boosting.
\newblock \emph{Journal of Computer and System Sciences}, 55\penalty0
  (1):\penalty0 119--139, 1997.

\bibitem[Helmbold et~al.(2000)Helmbold, Littlestone, and
  Long]{helmbold2000apple}
D.~P. Helmbold, N.~Littlestone, and P.~M. Long.
\newblock Apple tasting.
\newblock \emph{Information and Computation}, 161\penalty0 (2):\penalty0
  85--139, 2000.

\bibitem[Koc{\'a}k et~al.(2014)Koc{\'a}k, Neu, Valko, and
  Munos]{kocak2014efficient}
T.~Koc{\'a}k, G.~Neu, M.~Valko, and R.~Munos.
\newblock Efficient learning by implicit exploration in bandit problems with
  side observations.
\newblock In \emph{Advances in Neural Information Processing Systems}, pages
  613--621, 2014.

\bibitem[Littlestone and Warmuth(1994)]{LittlestoneWa94}
N.~Littlestone and M.~K. Warmuth.
\newblock The weighted majority algorithm.
\newblock \emph{Information and Computation}, 108:\penalty0 212--261, 1994.

\bibitem[Mannor and Shamir(2011)]{MS11}
S.~Mannor and O.~Shamir.
\newblock From bandits to experts: On the value of side-observations.
\newblock In J.~Shawe-Taylor, R.~Zemel, P.~Bartlett, F.~Pereira, and
  K.~Weinberger, editors, \emph{Advances in Neural Information Processing
  Systems 24}, pages 684--692. Curran Associates, Inc., 2011.

\bibitem[Shalev-Shwartz(2011)]{shalev2011online}
S.~Shalev-Shwartz.
\newblock Online learning and online convex optimization.
\newblock \emph{Foundations and Trends in Machine Learning}, 4\penalty0
  (2):\penalty0 107--194, 2011.

\bibitem[Vovk(1990)]{vo90}
V.~Vovk.
\newblock Aggregating strategies.
\newblock In \emph{Proceedings of the 3rd Annual Workshop on Computational
  Learning Theory}, pages 371--386, 1990.

\end{thebibliography}

\appendix


\section{Additional Proofs}
\label{sec:proofs}

\newcommand{\loss}{\ell}

\subsection{Proof of \cref{lem:regret-mw2}}

In order to prove our new regret bound for Hedge, we first state and prove the standard second-order regret bound for this algorithm. 
%
%
\begin{lemma} \label{lem:regret-mw1}
For any $\eta > 0$ and for any sequence $\ell_{1},\ldots,\ell_{T}$ of loss functions such that~$\ell_{t}(i) \ge -1/\eta$ for all $t$ and $i$, the probability vectors $q_{1},\ldots,q_{T}$ of \cref{eqn:mw} satisfy
\begin{align*}
    \sum_{t=1}^{T} \sum_{i \in V} q_{t}(i) \ell_{t}(i) &- \min_{k \in V} \sum_{t=1}^{T} \ell_{t}(k) 
\leq
    \frac{\ln K}{\eta}
    + \eta \sum_{t=1}^{T} \sum_{i \in V} q_{t}(i) \ell_{t}(i)^{2}~.
\end{align*}
\end{lemma}
\begin{proof}
The proof follows the standard analysis of exponential weighting schemes: let $w_t(i) = \exp\bigl(-\eta\sum_{s=1}^{t-1}\loss_s(i)\bigr)$ and let $W_t = \sum_{i \in V} w_t(i)$. Then $q_t(i) = w_t(i)/W_t$ and we can write
\begin{align*}
\frac{W_{t+1}}{W_t}
&\eq \sum_{i \in V} \frac{w_{t+1}(i)}{W_t}\\
&\eq \sum_{i \in V} \frac{w_{t}(i)\,\exp\bigl(-\eta\,\loss_{t}(i)\bigr)}{W_t}\\
&\eq \sum_{i \in V} q_{t}(i)\,\exp\bigl(-\eta\,\loss_{t}(i)\bigr)\\
&\leq \sum_{i \in V} q_{t}(i)\,\left(1 - \eta\loss_{t}(i) + \eta^2\loss_{t}(i)^2\right)
\qquad\quad \text{(using $e^{x} \le 1+x+x^2$ for all $x \le 1$)}\\
&\leq 1 - \eta\,\sum_{i \in V} q_{t}(i)\loss_{t}(i) + \eta^2\,\sum_{i \in V} q_{t}(i)\loss_{t}(i)^2~.
\end{align*}
Taking logs, using $\ln(1-x) \le -x$ for all $x \ge 0$, and summing over $t =
1, \ldots, T$ yields
\[
\ln \frac{W_{T+1}}{W_1} \leq \sum_{t=1}^T \sum_{i \in V} \left( -\eta\, q_{t}(i)\loss_{t}(i) +
\eta^2\,q_{t}(i)\loss_{t}(i)^2 \right)~.
\]
Moreover, for any fixed action $k$, we also have
\[
\ln \frac{W_{T+1}}{W_1} \geq \ln \frac{w_{T+1}(k)}{W_1} \eq -\eta\,\sum_{t=1}^T \loss_{t}(k) - \ln K~.
\]
Putting together and rearranging gives the result.
\end{proof}

We can now prove \cref{lem:regret-mw2}, restated here for the convenience of the reader.

\begin{repthm}[restated]{lem:regret-mw2}
Let $q_{1},\ldots,q_{T}$ be the probability vectors defined by \cref{eqn:mw} for a sequence of loss functions $\ell_{1},\ldots,\ell_{T}$ such that~$\ell_{t}(i) \ge 0$ for all $t=1,\dots,T$ and $i \in V$.
For each $t$, let $S_{t}$ be a subset of $V$ such that $\ell_{t}(i) \le 1/\eta$ for all $i \in S_{t}$.
Then, it holds that
\begin{align*}
	\sum_{t=1}^{T} \sum_{i \in V} q_{t}(i) \ell_{t}(i)
            	- \min_{k \in V} \sum_{t=1}^{T} \ell_{t}(k)
\le
	\frac{\ln K}{\eta} 
	+ \eta \sum_{t=1}^{T} \!\left(
	\sum_{i \in S_{t}} q_{t}(i) \bigl(1-q_{t}(i)\bigr) \ell_{t}(i)^{2} 
	+ \sum_{i \notin S_{t}} q_{t}(i) \ell_{t}(i)^{2} \!\right)
	\!\!.
\end{align*}
\end{repthm}
\begin{proof}
For all $t$, let $\bar{\ell}_{t} = \sum_{i \in S_{t}} p_{t}(i) \ell_{t}(i)$ for which $\bar{\ell}_{t} \le 1/\eta$ by construction.
Notice that executing Hedge on the loss vectors $\ell_{1},\ldots,\ell_{T}$ is equivalent to executing in on vectors $\ell'_{1},\ldots,\ell'_{T}$ with $\ell'_{t}(i) = \ell_{t}(i) - \bar{\ell}_{t}$ for all $i$.
Applying \cref{lem:regret-mw1} for the latter case (notice that $\ell'_{t}(i) \ge -1/\eta$ for all $t$ and $i$), we obtain
\begin{align*}
	\sum_{t=1}^{T} \sum_{i \in V} p_{t}(i) \ell_{t}(i) - \min_{k \in V} \sum_{t=1}^{T} \ell_{t}(k)
	&\eq \sum_{t=1}^{T} \sum_{i \in V} p_{t}(i) \ell'_{t}(i) - \min_{k \in V} \sum_{t=1}^{T} \ell'_{t}(k) \\
	&\leq \frac{\ln K}{\eta} + \eta \sum_{t=1}^{T} \sum_{i \in V} p_{t}(i) \ell'_{t}(i)^{2} \\
	&\eq \frac{\ln K}{\eta} + \eta \sum_{t=1}^{T} \sum_{i \in V} p_{t}(i) (\ell_{t}(i) - \bar{\ell}_{t})^{2}
	~.
\end{align*}
On the other hand, for all $t$,
\begin{align*}
\sum_{i \in S_{t}} p_{t}(i)  (\ell_{t}(i)-\bar{\ell}_{t})^{2}
&\eq \sum_{i \in S_{t}} p_{t}(i) \ell_{t}(i)^{2} 
	- \lrBig{\sum_{i \in S_{t}} p_{t}(i)\ell_{t}(i)}^{2} \\
&\leq \sum_{i \in S_{t}} p_{t}(i) \ell_{t}(i)^{2} 
	- \sum_{i \in S_{t}} p_{t}(i)^{2}\ell_{t}(i)^{2} \\
&\eq \sum_{i \in S_{t}} p_{t}(i) (1-p_{t}(i)) \ell_{t}(i)^{2}
\end{align*}
where the inequality follows from the non-negativity of the losses $\ell_{t}(i)$.
Also, since $\ell_{t}(i) > 1/\eta \ge \bar{\ell}_{t}$ for all $i \notin S_{t}$, we also have
\begin{align*}
	\sum_{i \notin S_{t}} p_{t}(i) (\ell_{t}(i)-\bar{\ell}_{t})^{2}
	\leq \sum_{i \notin S_{t}} p_{t}(i) \ell_{t}(i)^{2} ~.
\end{align*}
Combining the inequalities gives the lemma.
\end{proof}

\subsection{Proof of \cref{lem:alon}}

\begin{repthm}[restated]{lem:alon}
Let $G = (V,E)$ be a directed graph with $\abs{V}=K$, in which each node $i \in V$ is assigned a positive weight $w_{i}$.
Assume that $\sum_{i \in V} w_{i} \le 1$, and that $w_{i} \ge \eps$ for all $i \in V$ for some constant $0<\eps<\half$.
Then
\begin{align*}
	\sum_{i \in V} \frac{w_{i}}{w_{i}+\sum_{j \in \nin(i)} w_{j}}
\leq
	4 \alpha \ln \frac{4K}{\alpha\eps}
~,
\end{align*}
where $\alpha = \alpha(G)$ is the independence number of $G$.
\end{repthm}
\begin{proof}
Following the proof idea of \cite{Alonetal2013}, let $M = \ceil{2K/\eps}$ and introduce a discretization of the values $w_{1},\ldots,w_{T}$ such that $(m_{i}-1)/M \le w_{i} \le m_{i}/M$ for positive integers $m_{1},\ldots,m_{T}$.
Since each $w_{i} \ge \eps$, we have
$
	m_{i} 
	\geq M w_{i} 
	\geq \frac{2K}{\eps} \cdot \eps \eq 2K
$.
Hence, we obtain
\begin{align} \label{eq:wtom}
	\sum_{i \in V} \frac{w_{i}}{w_{i}+\sum_{j \in \nin(i)} w_{j}}
\eq
	\sum_{i \in V} \frac{m_{i}}{m_{i}+\sum_{j \in \nin(i)} m_{j} - K}
\leq
	2\sum_{i \in V} \frac{m_{i}}{m_{i}+\sum_{j \in \nin(i)} m_{j}}
~,
\end{align}
where the final inequality is true since $K \le \half m_{i} \le \half \bigl(m_{i}+\sum_{j \in \nin(i)} m_{j}\bigr)$.

Now, consider a graph $G' = (V',E')$ created from $G$ by replacing each node $i \in V$ with a clique $C_i$ over $m_i$ vertices, and connecting each vertex of $C_i$ to each vertex of $C_j$ if and only if the edge $(i,j)$ is present in $G$.
Then, the right-hand side of \cref{eq:wtom} equals $\sum_{i \in V'} \tfrac{1}{1+d_i}$, where $d_i$ is the in-degree of the vertex $i \in V'$ in the graph $G'$.
Applying Lemma~13 of \cite{Alonetal2013} to the graph $G'$, we can show that
\begin{align*}
	\sum_{i \in V} \frac{m_{i}}{m_{i}+\sum_{j \in \nin(i)} m_{j}}
\leq
	2\alpha \ln\lr{ 1 + \frac{\sum_{i \in V} m_{i}}{\alpha} }
\leq
	2\alpha \ln\lr{ 1 + \frac{M+K}{\alpha} }
\leq
	2\alpha \ln\frac{4K}{\alpha\eps}
~,
\end{align*}
and the lemma follows.
\end{proof}

\section{Proofs of Lower Bounds}
\label{sec:lower-proofs}


\subsection{Non-observable Feedback Graphs}

We first prove \cref{th:notobser}.
\begin{repthm}[restated]{th:notobser}
If $G=(V,E)$ is not observable and $\abs{V} \ge 2$, then for any player algorithm there exists a sequence of loss functions $\ell_{1},\ell_2,\ldots : V \mapsto [0,1]$ such that the player's expected regret is at least~$\tfrac{1}{4} T$.
\end{repthm}
\begin{proof}
Since $G$ is not observable, there exists a node with no incoming edges, say node $i = 1$. 
Consider the following randomized construction of loss functions $L_{1},L_2,\ldots : V \mapsto [0,1]$: draw $\chi \in \set{0,1}$ uniformly at random and set
\begin{align*}
	L_{t}(i) \eq \mycases
			{\chi}{if $i = 1$,}
			{\half}{if $i \ne 1$}
\qquad
    t=1,2,\dots
\end{align*}
Now fix some strategy of the player (which, without loss of generality, we may assume to be deterministic) and denote by $M$ the random number of times it chooses action $i=1$. Notice that the player's actions, and consequently $M$, are independent of the random variable $\chi$ since the player never observes the loss value assigned to action $i=1$. Letting $R_T$ denote the player's regret after $T$ rounds, it holds that $\E[R_T]$ (where expectation is taken with respect to the randomization of the loss functions) satisfies
\begin{align*}
	\E[R_T] 
&\eq 
	\thalf\,\E\bigl[ \tfrac{1}{2}M \mid \chi=1 \bigr]
	+ \thalf\,\E\bigl[ \tfrac{1}{2}(T-M) \mid \chi=0 \bigr]
\\
&\eq 
	\thalf\,\E\bigl[ \tfrac{1}{2}M + \tfrac{1}{2}(T-M) \bigr]
\\
&\eq
	\tfrac{1}{4} T
~.
\end{align*}
This implies that there exists a realization $\ell_{1},\ldots,\ell_{T}$ of the random functions for which the regret is at least $\tfrac{1}{8} T$, as claimed.
\end{proof}

\subsection{Weakly observable Feedback Graphs}

We now turn to prove our main lower bound for weakly observable graphs, stated in \cref{th:lower-weak}. 
\begin{repthm}[restated]{th:lower-weak}
If $G=(V,E)$ is weakly observable with $K = \abs{V} \ge 2$ and weak domination number $\del = \del(G)$, then for any randomized player algorithm and for any time horizon $T$ there exists a sequence of loss functions $\ell_{1},\ldots,\ell_{T} : V \mapsto [0,1]$ such that the player's expected regret is at least 
$
	\tfrac{1}{150} (\del/\ln^2\!{K})^{1/3} T^{2/3} .
$
\end{repthm}
Before proving the theorem, we recall the key combinatorial lemma it relies upon.
\begin{repthm}[restated]{lem:noga}
Let $G = (V,E)$ be a directed graph over $\abs{V} = n$ vertices, and let $W \subseteq V$ be a set of vertices whose minimal dominating set is of size $k$. 
Then, $W$ contains an independent set $U$ of size at least $\frac{1}{50}(k/\ln{n})$, with the property that any vertex of $G$ dominates at most $\ln{n}$ vertices of~$U$.
\end{repthm}


\begin{proof}[Proof of \cref{th:lower-weak}]
As the minimal dominating set of the weakly observable part of $G$ is of size~$\del$, \cref{lem:noga} says that $G$ must contain an independent set $U$ of $m \ge \del/(50\ln{K})$ weakly observable vertices, such that any $v \in V$ dominates at most $\ln{K}$ vertices of $U$.
For simplicity, we shall assume that $\del \ge 100\ln{K}$ which ensures that the set $U$ consists of at least $m \ge 2$ vertices;
a proof of the theorem for the (less interesting) case where $\del < 100\ln{K}$ is given after the current proof. 

Consider the following randomized construction of loss functions $L_{1},\ldots,L_{T} : V \mapsto [0,1]$: fix $\eps = m^{1/3}(32T\ln{K})^{-1/3}$, choose $\chi \in U$ uniformly at random and for all $t$ and $i$, and let the loss $L_{t}(i) \sim \mathsf{Ber}(\mu_{i})$ be a Bernoulli random variable with parameter
\begin{align*}
\forall ~ i \in V ~, \qquad
\mu_{i} \eq \mythreecases
	{\half-\eps}{if $i = \chi$,}
	{\half}{if $i \in U,~ i \ne \chi$,}
	{1}{$i \notin U$.}
\end{align*}
We refer to actions in $U$ as ``good'' actions (whose expected instantaneous regret is at most $\eps$), and to actions in $V \setminus U$ as ``bad'' actions (with expected instantaneous regret larger than $\half$).
Notice that $\nin(i) \subseteq V \setminus U$ for all good actions $i \in U$, since $U$ is an independent set of weakly observable vertices (that do not have self-loops).
In other words, in order to observe the loss of a good action in a given round, the player has to pick a bad action on that round.

Fix some strategy of the player which we assume to be deterministic (again, this is without loss of generality).
Up to a constant factor in the resulting regret lower bound, we may also assume that the strategy chooses bad actions at most $\eps T$ times with probability one (i.e., over any realization of the stochastic loss functions). 
Indeed, we can ensure this is the case by simply halting the player's algorithm once it chooses bad actions for more than $\eps T$ times, and picking an arbitrary good action in the remaining rounds; since the instantaneous regret of a good action is at most $\eps$, the regret of the modified algorithm is at most $3$ times larger than the regret of the original algorithm (the latter regret is at least $\half \eps T$, while the modification results in an increase of at most $\eps T$ in the regret).

Denote by $I_1,\ldots,I_{T}$ the sequence of actions played by the player's strategy throughout the game, in response to the loss functions $L_{1},\ldots,L_{T}$. 
For all $t$, let $Y_t$ be the vector of loss values observed by the player on round $t$; we think about $Y_{t}$ as being a full $K$-vector, with the unobserved values replaced by $-1$.
For all $i \in U$, let $M_{i}$ be the number of times the player picks the good action $i$, and $N_i$ be the number of times the player picks a bad action from $\nin(i)$.
Also, let $M$ be the total number of times the player picks a good action, and $N$ be the number of times he picks a bad action.
Notice that $\sum_{i \in U} N_i \le N \ln{K}$ as each vertex in $V \setminus U$ dominates at most $\ln{K}$ vertices of $U$ by construction. 
This, together with our assumption that $N \le \eps T$ with probability one (i.e., that the player picks bad actions for at most $\eps T$ times), implies that 
\begin{align} \label{eq:sum-Ni}
	\sum_{i \in U} N_i 
\leq
	\eps T \ln{K} 
~.
\end{align}

In order to analyze the amount of information on the value of $\chi$ the player obtains by observing the $Y_{t}$'s, we let $\F$ be the $\sigma$-algebra generated by the observed variables $Y_{1},\ldots,Y_{T}$, and define the conditional probability functions
$
	\Q^{i}(\cdot) = \Pr( \, \cdot \mid \chi = i )
$
over $\F$, for all $i \in U$.
Notice that under $\Q^{i}$, action $i$ is the optimal action. 
For technical purposes, we also let $\Q^{0}(\cdot)$ denote the fictitious probability function induced by picking $\chi = 0$; under this distribution, all good actions in $U$ have an expected loss equal to $\half$.
For two probability functions $\Q,\Q'$ over $\F$, we denote by
\begin{align*}
	\tv{\Q}{\Q'}
\eq
	\sup_{A \in \F} \abs{\Q(A)-\Q'(A)}
\end{align*}
the total variation distance between $\Q$ and $\Q'$ with respect to $\F$.
Then, we can bound the total variation distance between $\Q^{0}$ and each of the $\Q^{i}$'s in terms of the random variables $N_{i}$, as follows.
%
\begin{lemma*}
For each $i \in U$, we have
$
\tv{\Q^0}{\Q^i}
\le 
\eps \sqrt{2\E_{\Q^0}[N_i]}
.
$
\end{lemma*}
\begin{proof}
As an intermediate step, we first upper bound the KL-divergence between $\Q^{i}$ and $\Q^{0}$ in terms of the random variable $N_i$. 
Let $\Q^{j}_{t} = \Q^{j}(\,\cdot \mid Y_{1},\ldots,Y_{t-1})$ for all $j$. 
Notice that $\Q^{i}_{t}$ and $\Q^{0}_{t}$ are identical unless the player picked an action from $\nin(i)$ on round $t$. 
In this latter case, $\kl{ \Q^{0}_{t} }{ \Q^{i}_{t} }$ equals the KL-divergence between two Bernoulli random variables with biases $\half$ and $\half-\eps$, which is upper bounded by $4\eps^{2}$ for $\eps \le \tfrac{1}{4}$.%
\footnote{This KL-divergence equals 
$
\half \ln\frac{1/2}{1/2-\eps} + \half \ln\frac{1/2}{1/2+\eps}
=
\half \ln\lrbig{1+\frac{4\eps^2}{1-4\eps^2}}
\le
\half \cdot \frac{4\eps^2}{1-4\eps^2}
\le
4\eps^2,
$
where the last step is valid for $\eps \le \frac{1}{4}$.}
Thus, using the chain rule for relative entropy we may write
\begin{align*}
	\kl{ \Q^{0} }{ \Q^{i} }
&\eq
	\sum_{t=1}^{T} \kl{ \Q^{0}_{t} }{ \Q^{i}_{t} }
\\
&\eq
	\sum_{t=1}^{T} \Q^{0}\lrbig{ I_{t} \in \nin(i) }
	\cdot \kl{ \mathsf{Ber}(\thalf) }{ \mathsf{Ber}(\thalf-\eps) } 
\\
&\leq 
	4\eps^{2} \sum_{t=1}^{T} \Q^{0}\lrbig{ I_{t} \in \nin(i) }
\eq 
	4\eps^{2} \, \E_{\Q^{0}}[N_{i}]~.
\end{align*}
By Pinsker's inequality we have $\tv{\Q^0}{\Q^i} \le \sqrt{\thalf \kl{\Q^0}{\Q^i}}$, which gives the lemma.
\end{proof}
 
Averaging the lemma's inequality over $i \in U$, using the concavity of the square-root and recalling \cref{eq:sum-Ni}, we obtain
\begin{align} \label{eq:tvbound}
	\frac{1}{m} \sum_{i \in U} \tv{\Q^0}{\Q^i}
\leq
	\sqrt{\frac{2\eps^2}{m} \E_{\Q^0}\lrbra{ \sum_{i \in U} N_i }}
\leq
	\sqrt{\frac{2\eps^3}{m} T \ln{K}}
\eq
	\frac{1}{4}
~,
\end{align}
where the final equality follows from our choice of $\eps$.

We now turn to lower bound the player's expected regret.
Since the player incurs (at least) $\eps$ regret each time he picks an action different from $\chi$, his overall regret is lower bounded by $\eps(T-M_\chi)$, whence
\begin{align} \label{eq:RQi}
	\E[R_{T}]
\geq
	\frac{1}{m} \sum_{i \in U} \E\bigl[ \eps(T-M_\chi) \mid \chi=i \bigr]
\eq
	\eps T - \frac{\eps}{m} \sum_{i \in U} \E_{\Q^i}[ M_i ]
~.
\end{align}
In order to bound the sum on the right-hand side, note that 
$$
    \E_{\Q^{i}}[M_{i}] - \E_{\Q^{0}}[M_{i}]
\eq
    \sum_{t=1}^T \lr{ \Q^{i}(I_t = i) - \Q^{0}(I_t = i) }
\leq
    T \cdot \tv{\Q^0}{\Q^i}
~,
$$
and average over $i \in U$ to obtain
$$
	\frac{1}{m} \sum_{i \in U} \E_{\Q^{i}}[M_{i}]
\leq
	\frac{T}{m} \sum_{i \in U} \tv{\Q^0}{\Q^i} 
	+ \frac{1}{m} \E_{\Q^0}\lrbra{ \sum_{i \in U} M_i }
\leq
	\frac{1}{4} T + \frac{1}{m} T
\leq
	\frac{3}{4} T
~,
$$
where the last inequality is due to $m \ge 2$.
Combining this with \cref{eq:RQi} yields $\E[R_T] \ge \frac{1}{4} \eps T$, and plugging in our choice of $\eps$ gives
$$
	\E[R_T]
\geq
	\frac{1}{4} \lr{\frac{m}{32\ln{K}}}^{1/3} T^{2/3}
\geq
	\frac{\del^{1/3} T^{2/3}}{50\ln^{2/3}\!{K}}
~,
$$
which concludes the proof (recall the additional $\frac{1}{3}$-factor stemming from our simplifying assumption made earlier).
\end{proof}

The claim of the theorem for the case $\del < 100\ln{K}$, that remained unaddressed in the proof above, follows from a simpler lower bound that applies to weakly observable graphs of any size.

\begin{theorem} \label{th:weakobserv}
If $G=(V,E)$ is weakly observable and $\abs{V} \ge 2$, then for any player algorithm and for any time horizon $T$ there exists a sequence of loss functions $\ell_{1},\ldots,\ell_{T} : V \mapsto [0,1]$ such that the player's expected regret is at least $\tfrac{1}{8} T^{2/3}$.
\end{theorem}

\begin{proof}
First, we observe that any graph over less than $3$ vertices is either non-observable or strongly observable; in other words, any weakly observable graph has at least $3$ vertices, so $\abs{V} \ge 3$.
Now, if $G$ is weakly observable, then there is a node of $G$, say $i=1$, without a self-loop and without an incoming edge from (at least) one of the other nodes of the graph, say from $j=2$. 
Since $\abs{V} \ge 3$ and the graph is observable, $i=1$ has at least one incoming edge from a third node of the graph.

Consider the following randomized construction of loss functions $L_{1},\ldots,L_{T} : V \mapsto [0,1]$: fix $\eps = \thalf T^{-1/3}$, choose $\chi \in \set{-1,+1}$ uniformly at random and for all $t$ and $i$, let the loss $L_{t}(i) \sim \mathsf{Ber}(\mu_{i})$ be a Bernoulli random variable with parameter
\begin{align*}
	\mu_{i} \eq \mythreecases
			{\half-\eps \,\chi}{if $i = 1$,}
			{\half}{if $i = 2$,}
			{1}{otherwise.}
\end{align*}
Here, the ``good'' actions (whose expected instantaneous regret is at most $\eps$) are $i=1$ and $i=2$, and all other actions are ``bad'' actions (with expected instantaneous regret larger than $\half$).
%

Now, fix a deterministic strategy of the player and let the random variable $N_1$ be the number of times the player chooses a bad action from $\nin(1)$.
Define the conditional probability functions
$
	\Q^{1}(\cdot) = \Pr( \, \cdot \mid \chi = +1 )
$
and
$
	\Q^{2}(\cdot) = \Pr( \, \cdot \mid \chi = -1 )
$
where under $\Q^{i}$ action $i$ is the optimal action.
Also, define $\Q^0$ to be the fictitious distribution induced by setting $\chi=0$, under which the actions $i=1$ and $i=2$ both have an expected loss of $\half$.
Then, exactly as in the proof of \cref{th:lower-weak}, we can show that
$$
	\tv{ \Q^0 }{ \Q^i }
\leq 
	\eps \sqrt{2\E_{\Q^i}[N_1]} 
~,\qquad
	i=1,2
~.
$$
Averaging the two inequalities and using the concavity of the square root, we obtain
\begin{align} \label{eq:tvbound2}
    \tfrac{1}{2} \tv{ \Q^{0} }{ \Q^{1} }
	+ \tfrac{1}{2} \tv{ \Q^{0} }{ \Q^{2} }
\leq
	\eps \sqrt{\E_{\Q^1}[N_1]+\E_{\Q^2}[N_1]} 
\eq
	\eps \sqrt{2\E[N_1]} 
~,
\end{align}
where we have used the fact that $\Pr(\cdot) = \half \Q^1(\cdot) + \half \Q^2(\cdot)$.

%

We can now analyze the player's expected regret, again denoted by $\E[R_T]$.
Notice that if $\E[N_1] > \frac{1}{32}\eps^{-2}$, we have $\E[R_T] \ge \E[\tfrac{1}{2}N_1] > \frac{1}{64}\eps^{-2} = \tfrac{1}{8} T^{2/3}$ (since each action that reveals the loss of $i=1$ is a bad action whose instantaneous regret is at least $\half$), which gives the required lower bound.
Hence, we may assume that $\E[N_1] \le \frac{1}{32}\eps^{-2}$, in which case the right-hand side of \cref{eq:tvbound2} is bounded by $\frac{1}{4}$.
This yields an analogue of \cref{eq:tvbound}, from which we can proceed exactly as in the proof of \cref{th:lower-weak} to obtain that $\E[R_T] \ge \frac{1}{4} \eps T$.
Using our choice of $\eps$ gives the theorem.
%
\end{proof}

\subsection{Separation Between the Informed and Uninformed Models}

Finally, we prove our separation result for weakly observable time-varying graphs, which shows that the uninformed model is harder than the informed model (in terms of the dependence on the feedback structure) for weakly observable feedback graphs.

\begin{repthm}[restated]{th:weak-sep}
For any randomized player strategy in the uninformed feedback model, there exists a sequence of weakly observable graphs $G_{1},\ldots,G_{T}$ over a set $V$ of $K \ge 4$ actions with $\del(G_t)=\alpha(G_t)=1$ for all $t$, and a sequence of loss functions $\ell_{1},\ldots,\ell_{T} : V \mapsto [0,1]$, such that the player's expected regret is at least $\frac{1}{16} K^{1/3} T^{2/3}$.
\end{repthm}
\begin{proof}
As before, it is enough to demonstrate a randomized construction of weakly observable graphs $G_{1},\ldots,G_{T}$ and loss functions $L_{1},\ldots,L_{T}$ such that the expected regret of any deterministic algorithm is $\Omega(K^{1/3} T^{2/3})$.

The random loss functions $L_{1},\ldots,L_{T}$ are constructed almost identically to those used in the proof of \cref{th:weakobserv}; the only change is in the value of $\eps$, which is now fixed to $\eps = \tfrac{1}{4} (K/T)^{1/3}$.
In order to construct the random sequence of weakly observable graphs $G_{1},\ldots,G_{T}$, 
first pick nodes $J_{1},\ldots,J_{T}$ independently and uniformly at random from $V'=\set{3,\ldots,K}$.
Then, for each $t$, form the graph $G_{t}$ by taking the complete graph over $V$ (that includes all directed edges and self-loops) and removing all edges incoming to node $i=1$ (including its self-loop), except for the edge incoming from $J_{t}$.
In other words, the only way to observe the loss $L_{t}(1)$ of node $1$ on round $t$ is by picking the action $J_{t}$ on that round.
Notice that $G_{t}$ is weakly observable, as each of its nodes has at least one incoming edge, but there is a node (node $1$) which is not strongly observable.
Also, we have $\del(G_t)=1$ since $J_t$ dominates the entire graph, and  $\alpha(G_t)=1$ as any pair of vertices is connected by at least one directed edge.

We now turn to analyze the expected regret of any player on our construction; the analysis is very similar to that of \cref{th:weakobserv}, and we only describe the required modifications.
Fix any deterministic algorithm, and define the random variables $I_{1},\ldots,I_{T}$ and $N_1$ exactly as in the proof of \cref{th:weakobserv}.
In addition, define the distributions $\Q^0$, $\Q^1$, and $\Q^2$ as in that proof, for which we proved (recall \cref{eq:tvbound2}) that
\begin{align} \label{eq:tvbound3}
    \tfrac{1}{2} \tv{ \Q^{0} }{ \Q^{1} }
	+ \tfrac{1}{2} \tv{ \Q^{0} }{ \Q^{2} }
\leq
	\eps \sqrt{2\E[N_1]} 
~.
\end{align}
Now, define another random variable $N$ to be the number of times the player picked an action from $V'$ throughout the game.
Notice that in case $\E[N] > \tfrac{1}{4} K^{1/3}T^{2/3}$, we have $\E[R_T] \ge \E[\tfrac{1}{2}N] > \frac{1}{8} K^{1/3}T^{2/3}$ which implies the stated lower-bound on the expected regret. 
Hence, in what follows we assume that $\E[N] \le \tfrac{1}{4} K^{1/3}T^{2/3}$. 
Notice that for the graphs we constructed, $\Q(I_{t} = J_{t}) \le \frac{2}{K} \Q(I_{t} \in V')$ since $J_{t}$ is picked uniformly at random from $V'$ (and independently from $I_{t}$ because in the uninformed model $G_t$ is not known when $I_t$ is drawn) and since $K \ge 4$.
Summing this over $t=1,\ldots,T$, we obtain that $\E[N_1] \le \frac{2}{K} \E[N] \le \thalf (T/K)^{2/3}$, and with our choice of $\eps$ this shows that the right-hand size of  \cref{eq:tvbound3} is upper bounded by $\frac{1}{4}$. 
Again, continuing exactly as in the proof of \cref{th:lower-weak}, we finally get that $\E[R_T] \ge \frac{1}{4} \eps T$, and with our choice of $\eps$ this concludes the proof.
\end{proof}


\section{Tight Bounds for the Loopless Clique}
\label{s:clique}
We restate and prove \cref{th:blind}.
\begin{repthm}[restated]{th:blind}
For any sequence of loss functions $\ell_{1},\ldots,\ell_{T}$, where
$\ell_t:V \mapsto [0, 1]$, the expected regret of \cref{alg:alg}, with the loopless clique feedback graph and with parameters $\eta = \sqrt{(\ln K)/(2T)}$ and $\gamma=2\eta$, is upper-bounded by
$5\sqrt{T\ln K}$. 
\end{repthm}
\begin{proof}
Since $G$ is strongly observable, the exploration distribution $u$ is uniform on $V$. Fix any $i^{\st} \in V$. Notice that for any $i \in V$ we have $j \in \nin(i)$ for all $j \neq i$, and so $P_{t}(i) = 1-p_{t}(i)$. On the other hand, by the definition of $p_{t}$ and since $\eta = 2\gamma$ and $K \ge 2$, we have
$
	p_{t}(i) 
	\eq (1-\gamma) q_{t}(i) + \frac{\gamma}{K} 
	\le 1-\gamma + \tfrac{\gamma}{2}
	\eq 1-\eta
$,
so that $P_{t}(i) \ge \eta$. Thus, we can apply \cref{lem:regret-mw2} with $S_{t}=V$ to the vectors $\wh{\ell}_{1},\ldots,\wh{\ell}_{T}$ and take expectations,
\begin{align*}
    \E\left[ \sum_{t=1}^{T} \left(\sum_{i=1}^K q_{t}(i) \E_{t}[\wh{\ell}_{t}(i)]
    - \E_{t}[\wh{\ell}_{t}(i^{\st})]\right) \right] 
\le
    \frac{\ln K}{\eta} + \eta \, \sum_{t=1}^{T} \E\left[ 
	\sum_{i \in V} q_{t}(i) (1-q_{t}(i)) \E_{t}[\wh{\ell}_{t}(i)^{2}] 
	\right]~.
\end{align*}
Recalling \cref{eq:unbiased} and $P_{t}(i) = 1-p_{t}(i)$, we get
\begin{align*}
	\E\left[ \sum_{t=1}^{T} \sum_{i=1}^{k} q_{t}(i) \ell_{t}(i) \right]
            	- \sum_{t=1}^{T} \ell_{t}(i^{\st})
	\le
	\frac{\ln K}{\eta} 
	+ \eta \, \sum_{t=1}^{T} \E\left[ \sum_{i \in S} q_{t}(i) \frac{1-q_{t}(i)}{1-p_{t}(i)} \right]~.
\end{align*}
Finally, for the distributions $p_{t}$ and $q_{t}$ generated by the algorithm we note that
\[
	1-p_{t}(i) 
	\ge \bigl(1-\tfrac{\gamma}{K}\bigr) \bigl(1-q_{t}(i)\bigr) 
	\ge \tfrac{1}{2}\bigl(1-q_{t}(i)\bigr)
\]
where the last inequality holds since $K \ge 2$.
Hence,
\begin{align*}
\label{eq:clique3}
	\sum_{t=1}^{T} \sum_{i \in V} q_{t}(i) \frac{1-q_{t}(i)}{1-p_{t}(i)}
	\leq 2 \sum_{t=1}^{T} \sum_{i \in V} q_{t}(i)
	\leq 2T
	~.
\end{align*}
Combining this with \cref{eq:main2b} gives
\begin{align*}
	\E\left[ \sum_{t=1}^{T} \sum_{i \in V} p_{t}(i) \ell_{t}(i)
            	- \sum_{t=1}^{T} \ell_{t}(i^{\st}) \right]
\le
	\gamma T + \frac{\ln K}{\eta} + 2\eta T
\eq 
	\frac{\ln K}{\eta} + 4\eta T
	~,
\end{align*}
where we substituted our choice $\gamma = 2\eta$. Picking $\eta = \sqrt{(\ln K)/2T}$ proves the theorem.
\end{proof}


\newcommand{\bool}{\set{0,1}}
\section{Connections to Partial Monitoring}
\label{s:partial}
In online learning with partial monitoring the player is given a loss matrix $L$ over $[0,1]$ and a feedback matrix $H$ over a finite alphabet $\Sigma$. The matrices $L$ and $H$ are both of size $K \times M$, where $K$ is the number of player's actions and $M$ is the number of environment's actions. The environment preliminarily fixes a sequence $y_1,y_2,\dots$ of actions (i.e., matrix column indices) hidden from the player.\footnote{
The standard definition of partial monitoring \citep[see, e.g.,][Section~6.4]{cbl06} assumes a harder adaptive environment, where each action $y_t$ is allowed to depend on all of past player's actions $I_1,\dots,I_{t-1}$. However, the partial monitoring lower bounds of \citet[Theorem~13]{antos2013toward} and \citet[Theorem~3]{bartok2014partial} hold for our weaker notion of environment as well.
}
At each round $t=1,2,\dots$, the loss $\ell_t(I_t)$ of the player choosing action $I_t$ (i.e., a matrix row index) is given by the matrix entry $L(I_t,y_t) \in [0,1]$. The only feedback that the player observes is the symbol $H(I_t,y_t) \in \Sigma$; in particular, the column index $y_t$ and the loss value $L(I_t,y_t)$ remain both unknown. The player's goal is to control a notion of regret analogous to ours, where the minimization over $V$ is replaced by a minimization over the set of row indices, corresponding to the player's $K$ actions.

We now introduce a reduction from our online setting to partial monitoring for the special case of $\bool$-valued loss functions (note that our lower bounds still hold under this restriction, and so does our characterization of \cref{th:main}). Given a feedback graph $G$, we create a partial monitoring game in which the environment has a distinct action for each binary assignment of losses to vertices in $V$. Hence, $L$ and $H$ have $K$ rows and $M = 2^K$ columns, where the union of columns in $L$ is the set $\bool^K$. The entries of $H$ encode $G$ using any alphabet $\Sigma$ such that, for any row $i \in V$ and for any two columns $y \neq y'$, 
\begin{equation}
\label{eq:partmon}
    H(i,y) = H(i,y') \Leftrightarrow \Bigl\{\bigl(k,L(k,y)\bigr) \,:\, k \in \nout(i)\Bigr\} \equiv \Bigl\{\bigl(k,L(k,y')\bigr) \,:\, k \in \nout(i)\Bigr\}~.
\end{equation}
Note that this is a \textit{bona fide} reduction: given a partial monitoring algorithm $A$, we can define an algorithm $A'$ for solving any online learning problem with known feedback graph $G = (V,E)$ and $\{0,1\}$-valued loss functions. The algorithm $A'$ pre-computes a mapping from $\bigl\{\bigl(k,L(k,y)\bigr) \,:\, k \in \nout(i)\bigr\}$ for each $i \in V$ and for each $y=1,\dots,M$ to the alphabet $\Sigma$ such that \cref{eq:partmon} is satisfied. Then, at each round $t$, $A'$ asks $A$ to draw a row (i.e., a vertex of $V$) $I_t$ and obtains the feedback $\bigl\{\bigl(k,L(k,y_t)\bigr) \,:\, k \in \nout(I_t)\bigr\}$ from the environment. Finally, $A'$ uses the pre-computed mapping to obtain the symbol $\sigma_t \in \Sigma$ which is fed to $A$.

The minimax regret of partial monitoring games is determined by a set of observability conditions on the pair $(L,H)$. These conditions are expressed in terms of a canonical representation of $H$ as the set of matrices $S_i$ for $i \in V$. $S_i$ has a row for each distinct symbol $\sigma\in\Sigma$ in the $i$-th row of $H$, and $S_i(\sigma,y) = \ind{H(i,y) = \sigma}$ for $y=1,\dots,M$.
When cast to the class of pairs $(L,H)$ obtained from feedback graphs $G$ through the above encoding, the partial monitoring observability conditions of \citet[Definitions~5 and~6]{bartok2014partial} can be expressed as follows. Let $L(i,\cdot)$ be the column vector denoting the $i$-th row of $L$. Let also $\mathrm{rowsp}$ be the rowspace of a matrix and $\oplus$ be the cartesian product between linear spaces. Then
\begin{itemize}
\item $(L,H)$ is globally observable if for all pairs $i,j \in V$ of actions,
\[
    L(i,\cdot) - L(j,\cdot) \in \bigoplus_{k=1,\dots,K} \mathrm{rowsp}(S_k)
\]
\item $(L,H)$ is locally observable if for all pairs $i,j \in V$ of actions,
\[
    L(i,\cdot) - L(j,\cdot) \in \mathrm{rowsp}(S_i \oplus \mathrm{rowsp}(S_j)~.
\]
\end{itemize}
The characterization result for partial monitoring of \citet[Theorem~2]{bartok2014partial} states that the minimax regret is of order $\sqrt{T}$ for locally observable games and of order $T^{2/3}$ for globally observable games.
We now prove that the above encoding of feedback graphs $G$ as instances $(L,H)$ of partial monitoring games preserves the observability conditions. Namely, our encoding maps weakly (resp., strongly) observable graphs $G$ to globally (resp., locally) observable instances of partial monitoring. Combining this with our characterization result (\cref{th:main}) and the partial monitoring characterization result \citep[Theorem~2]{bartok2014partial}, we conclude that the minimax rates are preserved by our reduction.
\begin{claim}
\label{cl:partmon}
If $j \in \nout(i)$ then there exists a subset $\Sigma_0$ of rows of $S_i$ such that
\[
    L(j,\cdot) = \sum_{\sigma \in \Sigma_0} S_i(\sigma,\cdot)~.
\]
\end{claim}
\begin{proof}
Let $\Sigma_0$ to be the union of rows $S_i(\sigma_y,\cdot)$ such that $H(i,y) = \sigma_y$ and $L(j,y) = 1$ for some $y$. Each such row has a $1$ in position $y$ because $S_i(\sigma_y,y)=1$ holds by definition. Moreover, no such row has a $1$ in a position $y'$ where $L(j,y') = 0$. Indeed, combining $j \in \nout(i)$ with \cref{eq:partmon}, we get that $L(j,y') = 0$ implies $H(i,y') \neq \sigma_y$, which in turn implies $S_i(\sigma_y,y') = 0$. 
\end{proof}
\begin{theorem}
Any feedback graph $G$ can be encoded as a partial monitoring problem $(L,H)$ such that the observability conditions are preserved.
\end{theorem}
\begin{proof}
If $G$ is weakly observable, then for every $j \in V$ there is some $i \in V$ such that $j \in \nout(i)$. By \cref{cl:partmon}, $L(j,\cdot) \in \mathrm{rowsp}(S_i)$ and the global observability condition follows. If $G$ is strongly observable, then for any distinct $i,j \in V$ the subgraph $G'$ of $G$ restricted to the pair of vertices $i,j$ is weakly observable. By the previous argument, this implies that $L(i,\cdot) - L(j,\cdot) \in \mathrm{rowsp}(S_i) \oplus \mathrm{rowsp}(S_j)$ and the proof is concluded.
\end{proof}

\end{document}